\documentclass{article}

\usepackage[preprint]{neurips_2026}

\usepackage[utf8]{inputenc} 
\usepackage[T1]{fontenc}    
\usepackage{hyperref}       
\usepackage{url}            
\usepackage{booktabs}       
\usepackage{amsfonts}       
\usepackage{nicefrac}       
\usepackage{microtype}      
\usepackage{xcolor}         
\usepackage{multirow}
\usepackage{amsmath}
\usepackage{graphicx}

\usepackage{wrapfig}
\usepackage{amsthm}

\newtheorem{theorem}{Theorem}
\newtheorem{corollary}{Corollary}[theorem]

\title{World Models as Group Actions}

\author{
\begin{tabular}{c}
\textbf{Zijie~Wang}\textsuperscript{1,2}
\quad
\textbf{Wei~Zhang}\textsuperscript{3} \quad
\textbf{Weiming~Zhang}\textsuperscript{3} \quad
\textbf{Fanqi~Zhang}\textsuperscript{1}
\\[0.65em]
\textbf{Xiao~Tan}\textsuperscript{3} \quad
\textbf{Yipeng~Qin}\textsuperscript{4} \quad
\textbf{Guanbin~Li}\textsuperscript{1,2,5}\thanks{Correspondence to: liguanbin@mail.sysu.edu.cn}
\\[0.25em]
{\normalfont
\textsuperscript{1}Sun Yat-sen University \quad
\textsuperscript{2}Shenzhen Loop Area Institute \quad
\textsuperscript{3}Baidu Inc. \quad
\textsuperscript{4}Cardiff University
}
\\[0.25em]
{\normalfont 
\textsuperscript{5}Guangdong Key Laboratory of Big Data Analysis and Processing
}
\end{tabular}
}

\begin{document}

\maketitle

\begin{abstract}
Video world models have achieved strong visual realism, but this does not ensure that their dynamics are truly governed by actions. 
In this work, we argue that action faithfulness should be understood through the compositional structure of actions, which in many embodied settings follows a group structure (e.g., SE(2) for navigation). 
Based on this insight, we formalize action-conditioned world modeling as realizing a {\it group action} on the state space, providing a principled criterion for evaluating dynamics beyond visual quality.
To operationalize this framework, we propose a unified approach that enforces identity, inverse, and composition consistency via latent-space regularization with synthesized supervision, avoiding additional data collection. We further introduce two metrics: Group-Action Consistency (GAC) and Group-Action Robustness (GAR), to evaluate structural correctness and rollout stability. Extensive experimental results show that our method consistently improves both GAC and GAR in state-of-the-art video world models without degrading perceptual quality.
\end{abstract}

\section{Introduction}

Video world models are expected to serve as action-conditioned simulators for planning~\cite{song2023llm,li2026grhp,li2026embodied}, decision making~\cite{li2024embodieddm,zhang2025vlabench,liang2025large}, and embodied intelligence~\cite{liu2024meia,liu2025aligning,jiang2025beyond}.
Recent diffusion-based video models have greatly improved visual realism and long-horizon coherence~\cite{ho2022video,yang2024cogvideox,wan2025wan,bar2025nwm}.
However, visual realism alone does not ensure that a model behaves as a faithful simulator. A simulator must generate dynamics that {\it are governed by actions}, not merely correlated with them.

\begin{figure*}[!ht]
    \centering
    \includegraphics[width=\textwidth]{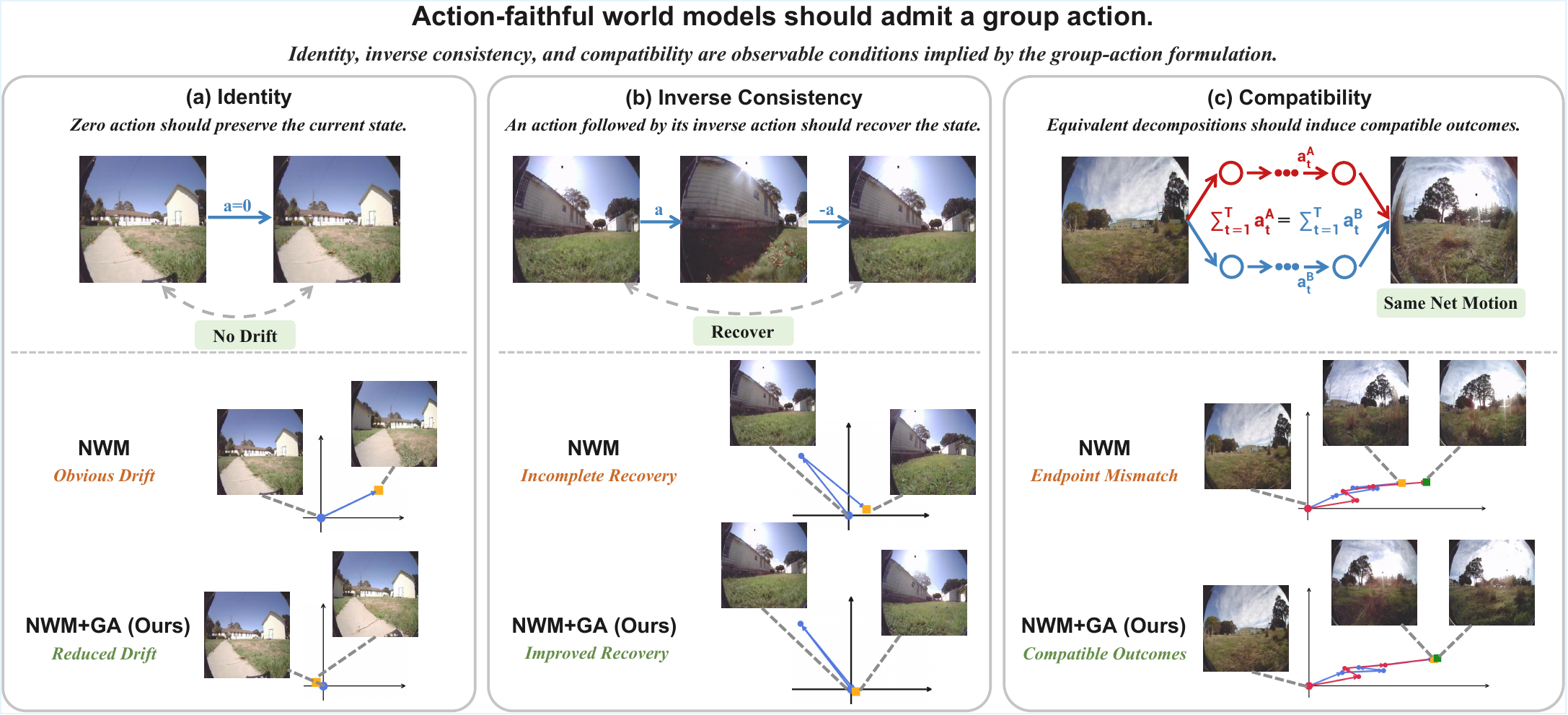}
    \vspace{-0.4cm}
    \caption{
    \textbf{Consequences of group-action-faithful dynamics.}
    Identity, inverse consistency, and compatibility are observable conditions implied by the group-action formulation.
    Representative examples illustrate that GA regularization reduces zero-action drift, improves forward--inverse recovery, and yields more compatible outcomes under equivalent action decompositions.
    }
    \vspace{-0.6cm}
    \label{fig:teaser-1}
\end{figure*}

In this work, we argue that the above-mentioned requirement can be understood through the compositional structure of actions. In many embodied settings, actions are not arbitrary inputs but correspond to transformations with well-defined algebraic properties. For instance, navigation actions can be interpreted as planar rigid motions, which form the Lie group $SE(2)$. Such structure imposes fundamental constraints on valid dynamics: a null action should leave the state unchanged, an action followed by its inverse should cancel, and different decompositions of the same transformation should produce consistent outcomes. These properties are intrinsic to how actions govern state evolution, rather than artifacts of any specific model. Fig.~\ref{fig:teaser-1} illustrates these consequences as observable conditions for action-faithful world modeling.

Building on this perspective, we propose a novel theoretical framework that formalizes \textbf{action-conditioned world modeling as realizing a \emph{group action} of the action space on the underlying state space}. 
It provides a general principle for action-conditioned world models whose actions induce structured transformations of the state, and we further instantiate it in the navigation setting of video world models, where the action space is grounded in $SE(2)$ and the resulting structure can be directly analyzed through generated rollouts.
This formulation introduces a new dimension for understanding and evaluating world models (beyond visual realism) by assessing whether their dynamics respect the underlying group action structure. Leveraging this perspective, we conduct a preliminary study on a recent state-of-the-art video world model, NWM~\cite{bar2025nwm}. Despite producing visually convincing rollouts, NWM exhibits clear violations of the group-action structure.

To address this gap, we introduce a novel framework for enforcing and evaluating action-faithful dynamics under the proposed group-action formulation. Specifically, it addresses the challenges of computational cost and lack of group-action supervision by moving the regularization from explicit state-space to efficient latent-space, where identity, inverse, and composition consistency are enforced directly on latent representations of model rollouts. To enable this, the method synthesizes supervision signals from existing trajectories by constructing zero-action, forward-inverse, and composition-equivalent action segments, avoiding the need for additional data collection. Finally, it defines two complementary evaluation metrics: Group-Action Consistency (GAC), which measures violations of structural properties in recovered state space, and Group-Action Robustness (GAR), which quantifies rollout stability under stochastic generation. Together, these components form a unified training and evaluation pipeline that operationalizes group-action structure in modern video world models.

Extensive experiments demonstrate that our framework consistently reduces both GAC and GAR errors across representative video world model backbones, while preserving perceptual generation quality.
In summary, our contributions include:
\begin{itemize}
    \item Conceptually, we formalize action-conditioned world modeling as a {\it group action}, providing a principled perspective that characterizes action faithfulness through identity, inverse, and composition consistency.
    \item Practically, we propose a unified framework that enforces these group-action constraints via latent-space regularization with synthesized supervision, along with corresponding evaluation metrics.
    \item Extensive experiments show that our approach consistently improves local consistency and long-horizon rollout stability in state-of-the-art video world models, while maintaining visual generation quality.
\end{itemize}
A comprehensive review of related work is provided in Appendix~\ref{sec:appendix_relatedwork}.
\section{Group-Action Formulation of World Models}
\label{sec:ga}

A central challenge in learning world models is ensuring that predictions are structurally consistent with the underlying transformations induced by actions. Standard approaches typically optimize one-step prediction accuracy, treating actions as independent inputs. As a result, the learned dynamics may fail to respect the compositional structure inherent in many physical systems.
Addressing this issue, we take a structured perspective and posit that, in many domains, actions possess an underlying algebraic structure, forming a group $G$ equipped with a composition law, identity element, and inverses. Rather than learning an unconstrained transition function, we require the world model $f_\theta : \mathcal S \times G \to \mathcal S$ to realize a \emph{group action} of $G$ on the state space:

\begin{theorem}
Let $G$ be a group with identity $e$, composition law $\circ$, and inverses $a^{-1}$. Suppose the world model $f_\theta : \mathcal S \times G \to \mathcal S$ satisfies, for all $s \in \mathcal S$ and $a_1,a_2,a\in G$,
\begin{align}
&\textnormal{(Compatibility)} \qquad
&f_\theta(f_\theta(s,a_1),a_2) = f_\theta(s,a_1\circ a_2), \\
&\textnormal{(Identity)} \qquad
&f_\theta(s,e) = s, \\
&\textnormal{(Inverse consistency)} \qquad
&f_\theta(f_\theta(s,a),a^{-1}) = s.
\end{align}
Then $f_\theta$ defines a \textbf{group action} of $G$ on $\mathcal S$. In particular, for any finite sequence $a_1,\dots,a_n\in G$,
\begin{equation}
    f_\theta(\cdots f_\theta(f_\theta(s,a_1),a_2)\cdots,a_n) = f_\theta(s,a_1\circ\cdots\circ a_n).
\end{equation}
\label{th:group_action}
\end{theorem}
Please see Appendix~\ref{sec:appendix_group_action} for the proof of Theorem~\ref{th:group_action}. 
Notably, in the video world models for embodied intelligence, actions often correspond to physical motions in the plane and can be naturally modeled as elements of the Lie group $SE(2)$: 
\begin{corollary}
If $G=SE(2)$, then a group-action-faithful model preserves the algebraic structure of planar rigid motions, so that the recovered motion induced by a rollout is consistent with the composed ego-motion, up to stochastic generation and state-estimation error.
\label{corollary:video_world_model}
\end{corollary}
In this case, the proposed formulation characterizes action faithfulness through consistency with planar rigid-motion composition.

\subsection{Instantiating Group Actions in Video World Models}
\label{sec:instantiation}

Video world models predict future observations conditioned on past observations and actions.
Given an initial observation $\mathbf{o}_0$ and an action sequence $\mathbf{a}_{1:T} = (\mathbf{a}_1,\dots,\mathbf{a}_T)$, the model generates a rollout
\begin{equation}
\mathbf{o}_{1:T} \sim p_{\theta}(\mathbf{o}_{1:T} \mid \mathbf{o}_0, \mathbf{a}_{1:T}),
\end{equation}
where each observation $\mathbf{o}_t$ is generated by an agent with state $\mathbf{s}_t$ navigating in a given scene at time $t$. According to Theorem~\ref{th:group_action} and Corollary~\ref{corollary:video_world_model}, this induces state dynamics of the agent in the form:
\begin{align}
\mathbf{s}_{t+1} &= f_{\theta}(\mathbf{s}_t, \mathbf{a}_t), 
\label{eq:state_transition}
\end{align}
where $\mathbf{s}_{t}=(R_{t}, \mathbf{p}_{t})$, $R_t$ is the rotation and $\mathbf{p}_t$ is the translation; $\mathbf{a}_t = (\Delta x_t, \Delta y_t, \Delta \theta_t)$ parameterizes a local planar ego-motion increment.
Note that for clarity, we suppress stochasticity in the generation process and focus on the deterministic transition dynamics, as our goal is to characterize the structural properties of the underlying dynamics.

\paragraph{Estimating Agent States from Observations.} 
We estimate the agent state $\mathbf{s}_t$ (Eq.~\ref{eq:state_transition}) from observations $\mathbf{o}_t$ using a pose estimator $\Phi(\cdot)$:
\begin{equation}
\mathbf{s}_{1:T}  = \Phi(\mathbf{o}_{1:T}),
\label{eq:observation_to_state}
\end{equation}
Based on the recovered states, we define a distance metric in the state space as:
\begin{equation}
d(\mathbf{s}_i, \mathbf{s}_j)
=
\|\mathbf{p}_i - \mathbf{p}_j\|_2
+
\alpha \, d_R(R_i, R_j),
\label{eq:state_distance}
\end{equation}
where $\mathbf{p}_i$ and $R_i$ denote the translation and rotation components of $\mathbf{s}_i$, respectively, $d_R(\cdot,\cdot)$ measures rotational discrepancy, and $\alpha$ balances the contributions of translation and rotation.

\paragraph{Instantiation of Group Action Conditions.}
Under the above ego-motion action parameterization, the composition law $\circ$ is operationally implemented by accumulating action increments over a segment, while the identity and inverse actions are represented by $\mathbf{e}=\mathbf{0}$ and $\mathbf{a}^{-1}=-\mathbf{a}$, respectively.
The relation between this operational form and exact $SE(2)$ composition is discussed in Appendix~\ref{app:local_composition_scope}.
Based on this structure, we instantiate the three group action conditions in Theorem~\ref{th:group_action} using the state distance defined in Eq.~\ref{eq:state_distance} as follows:
\begin{itemize}
    \item {\it Compatibility}: $d(\mathbf{s}^A_T, \mathbf{s}^B_T) = 0$, where initial state $\mathbf{s}^A_0 = \mathbf{s}^B_0$, action sequence $\mathbf{a}^A_{1:T} \neq \mathbf{a}^B_{1:T}$, and their accumulated increments satisfy $\sum^T_{t=1} \mathbf{a}^A_t = \sum^T_{t=1} \mathbf{a}^B_t$.
    \item {\it Identity}: $d(\mathbf{s}_e, \mathbf{s}) = 0$, where $\mathbf{s}_e = f_\theta(\mathbf{s},\mathbf{0})$.
    \item {\it Inverse consistency}: $d(\mathbf{s}_{inv}, \mathbf{s}) = 0$, where $\mathbf{s}_{inv}=f_\theta(f_\theta(\mathbf{s},\mathbf{a}),-\mathbf{a})$.
\end{itemize}

\subsection{Do Native Video World Models Respect Group Action Structure?}
\label{sec:preliminary_results}

Despite the group action structure outlined above in Sec.~\ref{sec:instantiation}, it remains unclear whether existing action-conditioned video world models, which primarily emphasize visual realism, respect this structure. To investigate this, we conduct preliminary experiments. As shown in Fig.~\ref{fig:teaser}, our results indicate that the state-of-the-art NWM~\cite{bar2025nwm} violates all three group action conditions (identity, inverse, and compatibility), even though the generated videos appear visually plausible. 
These findings reveal a clear gap between visual realism and dynamical consistency, motivating the development of methods to enforce such structure and systematic evaluation protocols to quantify adherence to group action properties in these video world models.

\begin{figure*}[t]
\centering
\includegraphics[width=\textwidth]{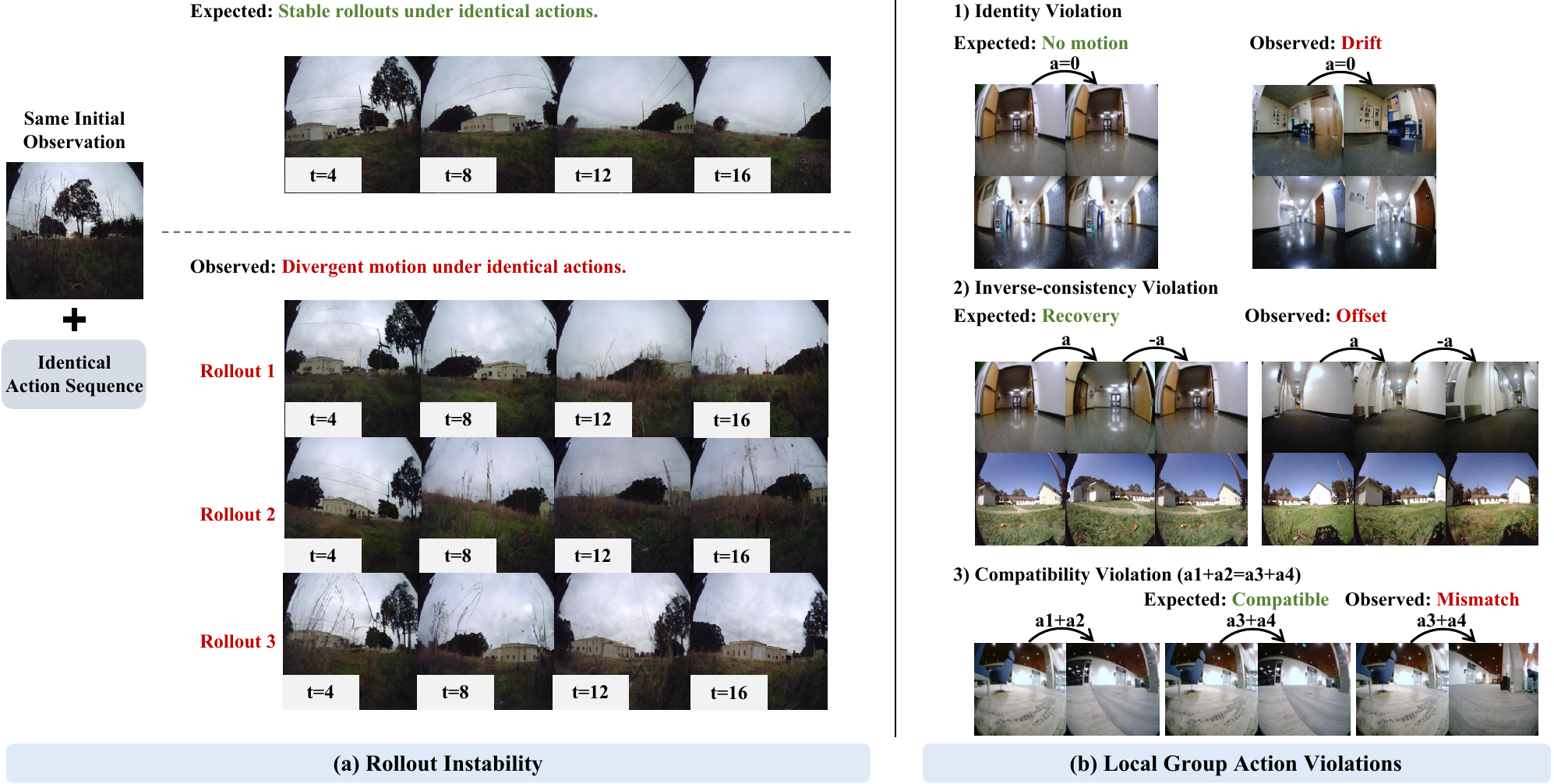}
\caption{
\textbf{Preliminary evidence of action inconsistency.}
(a) Repeated stochastic rollouts from the same initial observation and identical action sequence exhibit divergent motion.
(b) Controlled probes reveal violations of identity, inverse consistency, and compatibility, manifested as zero-action drift, incomplete inverse recovery, and endpoint mismatch under equivalent action decompositions.
}
\label{fig:teaser}
\vspace{-0.6cm}
\end{figure*}




\section{Methodology}
\label{sec:method}

The group-action formulation in Sec.~\ref{sec:ga} provides a principled state-space criterion for action-faithful world modeling, specifying the conditions required for rollouts to remain consistent with motion composition. In principle, these state-space conditions can be directly used as training objectives and evaluation metrics. However, doing so introduces two key challenges:
\begin{itemize}
    \item \textbf{Computational cost:} enforcing these constraints naively requires repeated rollout, decoding, and pose estimation, resulting in substantial computational overhead;
    \item \textbf{Lack of supervision:} there is no explicit training data corresponding to the identity, inverse, and compatibility conditions.
\end{itemize}
To address these challenges, we develop a unified framework comprising three novel components:
(1) we enforce group-action constraints as latent regularization, directly shaping the model's internal transition dynamics without requiring explicit supervision (Sec.~\ref{sec:latent_ga_supervision}); 
(2) we synthesize group-action training signals from existing video world model datasets to approximate supervision for identity, inverse, and compatibility conditions (Sec.~\ref{sec:ga_constraint_synthesis}); 
(3) in addition, we introduce an evaluation protocol in the recovered state space to systematically quantify the extent to which world models adhere to group action structure (Sec.~\ref{sec:ga_metrics}).


\subsection{Latent Group-Action Supervision}
\label{sec:latent_ga_supervision}

The group-action conditions in Sec.~\ref{sec:instantiation} are defined over agent states.
A straightforward training approach would therefore be to decode generated latents into image frames, estimate states with the pose estimator $\Phi$, and compute group-action errors on the recovered trajectories.
While conceptually faithful, this objective is impractical for optimization.
It would repeatedly invoke the decoder and an external pose estimation module inside the multi-step rollout loop, leading to heavy computation and brittle gradient propagation.
More critically, keeping this decode-and-estimate pipeline in the differentiable graph exceeds GPU memory for large video world models, making direct state-space optimization infeasible in our setting.

To address this issue, we propose a training objective that operates directly in the {\it latent space} of video world models.
Let $\phi_{\mathrm{enc}}$ denote the visual encoder.
For an observation $\mathbf{o}_t$, its latent representation is
\begin{equation}
    \mathbf{z}_t
    =
    \phi_{\mathrm{enc}}(\mathbf{o}_t).
\end{equation}
Accordingly, we define the latent group-action training objective as follows.


\paragraph{Latent Group-Action Training Objective.}
Let $D_z(\mathbf{x},\mathbf{y})=\|\mathbf{x}-\mathbf{y}\|_2^2$ denote the latent discrepancy. We design the three components of the training objective according to the group-action conditions (identity, inverse, and compatibility) as follows:
\begin{equation}
    \mathcal{L}_{\mathrm{id}}
    =
    D_z
    \bigl(
        \widehat{\mathbf{z}}^{\mathrm{id}},
        \mathbf{z}_t
    \bigr), \qquad
    \mathcal{L}_{\mathrm{inv}}
    =
    D_z
    \bigl(
        \widehat{\mathbf{z}}^{\mathrm{inv}},
        \mathbf{z}_t
    \bigr), \qquad
    \mathcal{L}_{\mathrm{comp}}
    =
    D_z
    \bigl(
        \widehat{\mathbf{z}}^{A},
        \widehat{\mathbf{z}}^{B}
    \bigr),
\label{eq:ga_latent_losses}
\end{equation}
where the latent endpoints $\widehat{\mathbf{z}}^{\mathrm{id}}$, $\widehat{\mathbf{z}}^{\mathrm{inv}}$, $\widehat{\mathbf{z}}^{A}$ and $\widehat{\mathbf{z}}^{B}$ are defined as:
\begin{equation}
\begin{aligned}
    \widehat{\mathbf{z}}^{\mathrm{id}}
    &=
    \operatorname{End}_{\theta}
    \bigl(
        \mathbf{z}_t,
        \mathbf{u}^{\mathrm{id}}
    \bigr),
    \\
    \widehat{\mathbf{z}}^{\mathrm{inv}}
    &=
    \operatorname{End}_{\theta}
    \bigl(
        \mathbf{z}_t,
        \mathbf{u}^{\mathrm{inv}}
    \bigr),
    \\
    \widehat{\mathbf{z}}^{A},
    \widehat{\mathbf{z}}^{B}
    &=
    \operatorname{End}_{\theta}
    \bigl(
        \mathbf{z}_t,
        \mathbf{u}^A
    \bigr),
    \operatorname{End}_{\theta}
    \bigl(
        \mathbf{z}_t,
        \mathbf{u}^B
    \bigr).
\end{aligned}
\label{eq:ga_latent_endpoints}
\end{equation}
where $\operatorname{End}_{\theta}(\mathbf{z}_t, \cdot)$ denotes the latent endpoint obtained by rolling out the model from $\mathbf{z}_t$ under a given action segment $\mathbf{u}$ using its native transition dynamics.
The overall group-action objective is therefore defined as:
\begin{equation}
    \mathcal{L}_{\mathrm{GA}}
    =
    \lambda_{\mathrm{id}}\mathcal{L}_{\mathrm{id}}
    +
    \lambda_{\mathrm{inv}}\mathcal{L}_{\mathrm{inv}}
    +
    \lambda_{\mathrm{comp}}\mathcal{L}_{\mathrm{comp}},
\end{equation}
The revised diffusion training objective becomes:
\begin{equation}
    \mathcal{L}
    =
    \mathcal{L}_{\mathrm{diff}}
    +
    \lambda_{\mathrm{GA}}\mathcal{L}_{\mathrm{GA}}.
\label{eq:training_objective}
\end{equation}
Eq.~\ref{eq:training_objective} describes the full objective.
To reduce the computation and memory cost of multiple free-running rollouts, we sample one constraint type from
$\{\mathrm{id},\mathrm{inv},\mathrm{comp}\}$ in each batch and optimize the corresponding stochastic estimate of Eq.~\ref{eq:training_objective}.
Here, $\mathcal{L}_{\mathrm{diff}}$ preserves visual realism, while $\mathcal{L}_{\mathrm{GA}}$ regularizes latent dynamics.
Existing video world model datasets do not contain paired trajectories that explicitly realize zero-motion preservation, forward--inverse cancellation, or alternative decompositions of the same local motion.
Therefore, the corresponding group-action relations are synthesized from existing trajectories.

\paragraph{Remark.} The latent variables represent the internal rollout states from which future observations are generated. Constraining their evolution therefore regularizes the model transition before visual decoding. Although the group-action conditions are evaluated in recovered state space, their latent-space counterparts are constructed from the same algebraic relations: identity preservation, inverse cancellation, and local composition consistency. The latent objective should therefore be viewed as an efficient surrogate for state-space regularization, rather than an exact replacement. It encourages the internal rollout dynamics to follow the same structural relations that are later measured on recovered trajectories, while avoiding the memory-prohibitive decode-and-estimate loop during training.

\subsection{Group-Action Supervision Synthesis}
\label{sec:ga_constraint_synthesis}

To address the lack of supervision described above, one approach is to collect additional trajectories that explicitly realize pauses, forward--inverse cycles, or multiple decompositions of the same local displacement. While such data would provide direct supervision, it is costly to acquire and difficult to scale.
We therefore propose a novel strategy to synthesize group-action supervision from existing training trajectories by leveraging the algebraic operations of the action space.
Specifically, given a trajectory from an existing dataset consisting of observations and actions:
\begin{equation}
    (\mathbf{o}_{1:T}, \mathbf{a}_{1:T-1}),
\end{equation}
then, for any $1 \leq t < T-1$, we construct action segments $\mathbf{u}^{\mathrm{id}}_{1:l}$, $\mathbf{u}^{\mathrm{inv}}_{1:2l}$ and ($\mathbf{u}^{A}_{1:l}$, $\mathbf{u}^{B}_{1:l}$), where length $l$ satisfies $t + l - 1 \leq T$, used in Eq.~\ref{eq:ga_latent_endpoints} as follows.

\paragraph{Identity Action Segment.}
The action segment for the identity condition should induce no state change and is implemented as a zero-action segment:
\begin{equation}
    \mathbf{u}^{\mathrm{id}}_{1:l}
    =
    (\mathbf{0},\ldots,\mathbf{0}).
\end{equation}
As shown in Eq.~\ref{eq:ga_latent_losses}, this condition enforces that executing the zero-action segment $\mathbf{u}^{\mathrm{id}}_{1:l}$ preserves the latent state, keeping the endpoint close to $\mathbf{z}_t$.


\paragraph{Inverse Action Segment.}
The inverse condition requires a motion followed by its inverse to recover the starting state.
Therefore, let $\mathbf{u}_{1:l}=(\mathbf{a}_{t},\ldots,\mathbf{a}_{t+l-1})$ be an action segment from the original trajectory, we construct the inverse action segment as a forward--inverse cycle:
\begin{equation}
\mathbf{u}^{\mathrm{inv}}_{1:2l}
    =
    (\mathbf{u}_{1:l}, \hat{\mathbf{u}}_{1:l})
    =
    (\mathbf{a}_{t},\ldots,\mathbf{a}_{t+l-1},-\mathbf{a}_{t+l-1},\ldots,-\mathbf{a}_{t}),
\end{equation}
As shown in Eq.~\ref{eq:ga_latent_losses}, the endpoint should also be close to $\mathbf{z}_t$, as the inverse action segment cancels the forward action.

\paragraph{Compatibility Action Segment.}
The compatibility condition requires that different decompositions of the same intended local motion yield compatible endpoints.
Let $\mathbf{u}^A_{1:l}= (\mathbf{u}^A_{1},\ldots,\mathbf{u}^A_{l})=(\mathbf{a}_{t},\ldots,\mathbf{a}_{t+l-1})$ denote an action segment from the original trajectory. We construct an alternative decomposition $\mathbf{u}^B_{1:l}=(\mathbf{u}^B_{1},\ldots,\mathbf{u}^B_{l})$ that:
\begin{equation}
    \mathbf{u}^B_{i}
    =
    w_i
    \sum_{j=1}^{l}\mathbf{u}^A_{j},
    \qquad
    i=1,\ldots,l .
\label{eq:alternative_decomposition}
\end{equation}
where $\mathbf{w} = (w_1, \ldots, w_l) \sim \mathrm{Dir}(\alpha)$ are non-negative weights sampled from a Dirichlet distribution. By construction, the accumulated increments are identical:
\begin{equation}
    \sum_{i=1}^{l}\mathbf{u}^A_{i}
    =
    \sum_{i=1}^{l}\mathbf{u}^B_{i}.
\label{eq:cumulative_equivalence}
\end{equation}
Accordingly, the loss in Eq.~\ref{eq:ga_latent_losses} enforces that rollouts from the same starting latent under $\mathbf{u}^A_{1:l}$ and $\mathbf{u}^B_{1:l}$ produce matching endpoints.

\subsection{Group-Action Metrics}
\label{sec:ga_metrics}

While group-action conditions are enforced in latent space during training, the resulting dynamics are ultimately assessed through motion in generated videos. We therefore evaluate in the recovered state space.
Following the group-action conditions defined in Sec.~\ref{sec:instantiation}, we introduce two complementary metrics: \emph{Group-Action Consistency} (GAC), which measures violations of these conditions, and \emph{Group-Action Robustness} (GAR), which assesses rollout stability under stochastic generation.

\paragraph{Group-Action Consistency.}
Group-Action Consistency (GAC) measures the extent to which the recovered motion induced by generated rollouts satisfies the group-action conditions.
Specifically, we construct a state-space analogue of Eq.~\ref{eq:ga_latent_losses} and evaluate each condition under a fixed family of controlled probe configurations.
For identity and inverse consistency, a probe configuration is indexed by $(k,l)$, where $k$ denotes the number of inserted or evaluated segments and $l$ denotes the length of each segment.
For local composition consistency, the probe configuration is indexed by the local window length $l$.
Within each configuration, errors are averaged over fixed valid starting positions and evaluation sequences, yielding configuration-level errors
$\Delta_{\mathrm{id}}^{(k,l)}$, $\Delta_{\mathrm{inv}}^{(k,l)}$, and $\Delta_{\mathrm{comp}}^{(l)}$.
Detailed definitions of these configuration-level errors are provided in Appendix~\ref{app:gac_eval_protocol}.
The reported component errors are obtained by averaging over the evaluated probe configurations:
\begin{equation}
    \bar{\Delta}_{\mathrm{id}}
    =
    \frac{1}{|\mathcal{K}_{\mathrm{id}}|}
    \sum_{(k,l)\in\mathcal{K}_{\mathrm{id}}}
    \Delta_{\mathrm{id}}^{(k,l)},
    \;
    \bar{\Delta}_{\mathrm{inv}}
    =
    \frac{1}{|\mathcal{K}_{\mathrm{inv}}|}
    \sum_{(k,l)\in\mathcal{K}_{\mathrm{inv}}}
    \Delta_{\mathrm{inv}}^{(k,l)},
    \;
    \bar{\Delta}_{\mathrm{comp}}
    =
    \frac{1}{|\mathcal{L}_{\mathrm{comp}}|}
    \sum_{l\in\mathcal{L}_{\mathrm{comp}}}
    \Delta_{\mathrm{comp}}^{(l)} .
\label{eq:gac_components}
\end{equation}
These terms measure identity drift, inverse recovery error, and composition mismatch, respectively.
The aggregated GAC error is defined as:
\begin{equation}
    \mathcal{E}_{\mathrm{GAC}}
    =
    \frac{1}{3}
    \left(
    \bar{\Delta}_{\mathrm{id}}
    +
    \bar{\Delta}_{\mathrm{inv}}
    +
    \bar{\Delta}_{\mathrm{comp}}
    \right),
\label{eq:gac_error}
\end{equation}
where a lower $\mathcal{E}_{\mathrm{GAC}}$ indicates stronger adherence to the group-action conditions.

\paragraph{Group-Action Robustness.}
Group-Action Robustness (GAR) measures the stability of motion under stochastic rollout generation.
Given the same initial observation and identical action sequence, we sample $N$ stochastic rollouts and recover their state trajectories
$\{\mathbf{s}^{(i)}_{1:T}\}_{i=1}^{N}$.
The GAR error is then defined as:
\begin{equation}
    \mathcal{E}_{\mathrm{GAR}}
    =
    \frac{2}{N(N-1)}
    \sum_{1\leq i<j\leq N}
    \frac{1}{T}
    \sum_{t=1}^{T}
    d
    \bigl(
        \mathbf{s}^{(i)}_{t},
        \mathbf{s}^{(j)}_{t}
    \bigr).
\label{eq:gar_error}
\end{equation}
where lower $\mathcal{E}_{\mathrm{GAR}}$ indicates that repeated stochastic rollouts remain more consistent under the same action sequence.


Additional details on free-running latent rollout, constraint synthesis and sampling, action-space approximations, and the practical per-batch training objective are provided in Appendix~\ref{sec:appendix_ga_training}.

\section{Experiments}
\label{sec:exp}

\subsection{Experimental Setup}
\label{sec:exp_setup}

Our experiments are conducted on the RECON~\cite{shah2021recon}, SCAND~\cite{karnan2022scand}, and HuRoN~\cite{hirose2023sacson} datasets, following the preprocessing protocol of NWM~\cite{bar2025nwm}. 
We report the main results on RECON, with additional results on SCAND and HuRoN provided in Appendix~\ref{app:additional_results}. 
We use NWM~\cite{bar2025nwm} as the primary testbed and include DIAMOND~\cite{alonso2024diamond} as an additional baseline. 
For GAC evaluation, we use a fixed set of controlled probes shared by all methods.
For identity and inverse probes, configurations are indexed by the number of local segments $k$ and the segment length $l$.
Increasing $k$ tests whether violations accumulate when the same local relation is applied repeatedly, while increasing $l$ tests whether violations grow with the temporal extent of each segment.
We vary one factor while fixing the other, using $k,l\in\{1,3,5\}$ for the main RECON analysis.
For local composition probes, we vary the window length with $l\in\{2,4,6\}$, since the probe compares two decompositions of the same local motion.
Within each configuration, errors are averaged over fixed valid starting positions and evaluation sequences.
For GA training, the local rollout horizon is sampled from $l\sim\mathcal{U}\{1,\ldots,L\}$ with $L=4$ by default, and one group-action constraint type is sampled per batch for efficiency.
Evaluation is conducted using the metrics defined in Sec.~\ref{sec:ga_metrics}.
Additional details of dataset statistics, model adaptation, training hyper-parameters, pose recovery, GAC probe configurations, GAR rollout settings, and metric aggregation procedures are provided in Appendix~\ref{app:exp_details}.

\begin{table*}[!ht]
\caption{
Group-Action Consistency (GAC) errors on the RECON dataset.
Columns 2--4 report the individual error terms, and the last column reports the aggregated GAC error $E_{\mathrm{GAC}}$.
``+GA'' denotes applying the proposed GA regularization.
$E_{\mathrm{GAC}}$ is computed from the mean component errors.
}
\label{tab:gac_aggregate_recon}
\centering
\setlength{\tabcolsep}{7pt}
\renewcommand{\arraystretch}{1.08}
\begin{tabular}{lcccc}
\toprule
Method
& $\bar{\Delta}_{\mathrm{id}} \downarrow$
& $\bar{\Delta}_{\mathrm{inv}} \downarrow$
& $\bar{\Delta}_{\mathrm{comp}} \downarrow$
& $\mathcal{E}_{\mathrm{GAC}} \downarrow$ \\
\midrule

DIAMOND~\cite{alonso2024diamond}
& $2.47 \pm 1.16$
& $2.38 \pm 1.08$
& $1.12 \pm 0.70$
& $1.99$ \\

DIAMOND + GA
& $2.25 \pm 1.03$
& $2.16 \pm 0.97$
& $0.86 \pm 0.55$
& $1.76$ \\

NWM~\cite{bar2025nwm}
& $2.10 \pm 0.99$
& $2.29 \pm 0.99$
& $0.79 \pm 0.52$
& $1.73$ \\

NWM + GA
& $\mathbf{1.95 \pm 0.86}$
& $\mathbf{1.95 \pm 0.84}$
& $\mathbf{0.60 \pm 0.40}$
& $\mathbf{1.50}$ \\



\bottomrule
\end{tabular}
\end{table*}

\subsection{Group-Action Consistency Results}
\label{sec:exp_gac}

Table~\ref{tab:gac_aggregate_recon} reports the Group-Action Consistency (GAC) errors on RECON.
Compared with their GA-regularized counterparts, both baseline models show consistently larger errors across identity, inverse, and local composition consistency.
This indicates that strong visual rollout quality does not by itself guarantee adherence to the underlying group-action structure.
Applying GA regularization reduces the aggregate GAC error for both DIAMOND and NWM, with gains observed across all three components.
These results suggest that the proposed objective improves action-faithful dynamics beyond a single backbone architecture.
Fig.~\ref{fig:asc_trends} provides a probe-wise view of the same effect.
Across most probe settings, GA regularization shifts the error curves downward and reduces the standard deviation indicated by the marker size.
The improvement is particularly clear for inverse and local composition consistency, where errors tend to increase when the model must compose multiple action-conditioned transitions.
For identity consistency, the curve is less monotonic because repeated or longer zero-action segments can reinforce the stationary condition, but larger settings impose a stricter preservation requirement across more pause locations or longer pause durations.
Additional GAC results and probe-wise analysis are provided in Appendices~\ref{app:gac_more} and~\ref{app:gac_probe_analysis}.

\begin{figure}[!ht]
    \centering
    \includegraphics[width=\columnwidth]{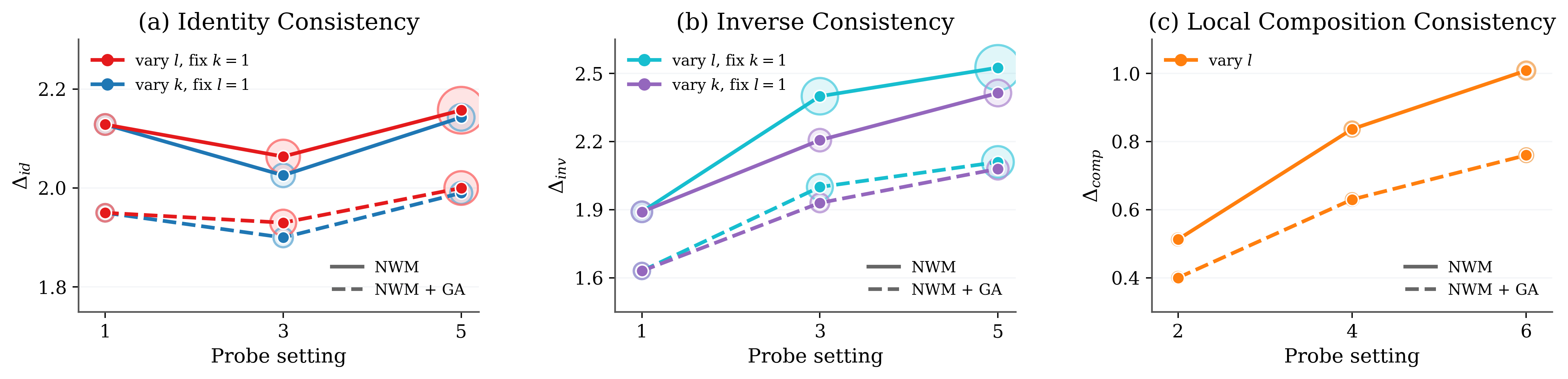}
    \vspace{-0.6cm}
    \caption{
    Group-Action Consistency (GAC) errors on RECON under different probe configurations.
    For identity and inverse probes, $k$ and $l$ denote the number and length of probed segments; for composition probes, $l$ denotes the window length.
    Markers indicate mean errors, with size proportional to the standard deviation.
    Lower values indicate stronger group-action consistency.
    }
    \label{fig:asc_trends}
    \vspace{-0.4cm}
\end{figure}

\begin{table*}[!ht]
\caption{
Group-Action Robustness (GAR) errors on the RECON dataset.
``Aligned'' removes global pose drift before computing trajectory error, while ``Non-aligned'' measures raw trajectory deviation.
}
\label{tab:trajectory_consistency_recon_all}
\centering
\setlength{\tabcolsep}{7pt}
\renewcommand{\arraystretch}{1.08}
\begin{tabular}{lcccc}
\toprule
\multirow{2}{*}{Method} & \multicolumn{2}{c}{16 Frames} & \multicolumn{2}{c}{64 Frames} \\
\cmidrule(lr){2-3}\cmidrule(lr){4-5}
& Aligned $\downarrow$ & Non-aligned $\downarrow$ & Aligned $\downarrow$ & Non-aligned $\downarrow$ \\
\midrule

DIAMOND~\cite{alonso2024diamond}
& $0.44 \pm 0.36$
& $0.89 \pm 0.58$
& $0.93 \pm 0.67$
& $2.48 \pm 1.91$ \\

DIAMOND + GA
& $0.37 \pm 0.29$
& $0.74 \pm 0.49$
& $0.76 \pm 0.55$
& $2.03 \pm 1.46$ \\

NWM~\cite{bar2025nwm}
& $0.32 \pm 0.26$
& $0.65 \pm 0.41$
& $0.62 \pm 0.44$
& $1.52 \pm 1.27$ \\

NWM + GA
& $\mathbf{0.26 \pm 0.20}$
& $\mathbf{0.49 \pm 0.35}$
& $\mathbf{0.46 \pm 0.33}$
& $\mathbf{1.28 \pm 0.82}$ \\


\bottomrule
\end{tabular}
\vspace{-0.4cm}
\end{table*}

\subsection{Group-Action Robustness Results}
\label{sec:exp_gar}

Table~\ref{tab:trajectory_consistency_recon_all} reports GAR errors on RECON, with additional results on SCAND and HuRoN provided in Appendix~\ref{app:gar_more}.
The errors are computed from repeated rollouts $(N=5)$ generated from the same initial observation and action sequence drawn from the test set.
It can be observed that the baseline NWM exhibits clear horizon-dependent degradation. In particular, its non-aligned error increases from $0.65$ at 16 frames to $1.52$ at 64 frames, indicating that stochastic variations accumulate into substantial global trajectory drift over longer rollouts. 
In contrast, GA training consistently reduces GAR error across both horizons and alignment settings, with more pronounced improvements at 64 frames where accumulated transition errors are larger. 
The reduction in non-aligned error indicates that stochastic rollouts become more consistent in their global motion under the same action sequence, rather than only after trajectory alignment.
Overall, these results demonstrate that the proposed GA regularization enhances robustness to stochastic rollout variability, especially over longer horizons.

\begin{table*}[!ht]
\vspace{-0.4cm}
\caption{
Ablation of group-action constraints on RECON using NWM as the base model. Single-loss rows apply only the indicated group-action constraint. Full GA uses all three constraints. GAC reports local group-action consistency, while GAR reports rollout-level robustness at 64 frames. Lower is better.
}
\label{tab:ga_ablation_recon}
\centering
\setlength{\tabcolsep}{5.5pt}
\renewcommand{\arraystretch}{1.06}
\resizebox{\textwidth}{!}{
\begin{tabular}{lcccccc}
\toprule
\multirow{2}{*}{Method}
& \multicolumn{4}{c}{GAC}
& \multicolumn{2}{c}{GAR 64 Frames} \\
\cmidrule(lr){2-5}\cmidrule(lr){6-7}
& $\bar{\Delta}_{\mathrm{id}} \downarrow$
& $\bar{\Delta}_{\mathrm{inv}} \downarrow$
& $\bar{\Delta}_{\mathrm{comp}} \downarrow$
& $\mathcal{E}_{\mathrm{GAC}} \downarrow$
& Aligned $\downarrow$
& Non-aligned $\downarrow$ \\
\midrule

Baseline
& $2.10 \pm 0.99$
& $2.29 \pm 0.99$
& $0.79 \pm 0.52$
& $1.73$
& $0.62 \pm 0.44$
& $1.52 \pm 1.27$ \\

$\mathcal{L}_{\mathrm{id}}$
& $1.93 \pm 0.87$
& $2.18 \pm 0.96$
& $0.79 \pm 0.50$
& $1.62$
& $0.57 \pm 0.41$
& $1.44 \pm 1.16$ \\

$\mathcal{L}_{\mathrm{inv}}$
& $2.03 \pm 0.94$
& $1.99 \pm 0.85$
& $0.73 \pm 0.48$
& $1.58$
& $0.53 \pm 0.38$
& $1.38 \pm 1.06$ \\

$\mathcal{L}_{\mathrm{comp}}$
& $2.06 \pm 0.95$
& $2.08 \pm 0.90$
& $0.62 \pm 0.41$
& $1.59$
& $0.50 \pm 0.35$
& $1.33 \pm 0.94$ \\

Full GA
& $\mathbf{1.95 \pm 0.86}$
& $\mathbf{1.95 \pm 0.84}$
& $\mathbf{0.60 \pm 0.40}$
& $\mathbf{1.50}$
& $\mathbf{0.46 \pm 0.33}$
& $\mathbf{1.28 \pm 0.82}$ \\

\bottomrule
\end{tabular}
}
\end{table*}

\subsection{Ablation Analysis}
\label{sec:exp_ablation}

Table~\ref{tab:ga_ablation_recon} presents an ablation of the three group-action conditions defined in Eq.~\ref{eq:ga_latent_losses}. Each individual condition improves over the NWM baseline in both GAC and GAR, indicating that identity preservation, inverse cancellation, and compatibility enforcement each address distinct failure modes in world model dynamics. 
The full objective achieves the best overall performance, suggesting that the three conditions are complementary. 
Notably, the GAC results show that improvements are not strictly confined to the directly targeted component: optimizing inverse or compatibility also reduces other errors, reflecting coupling among the conditions through shared rollout dynamics. 
The GAR ablation results further show that group-action consistency improvements translate into more stable repeated rollouts. Among the individual terms, the compatibility term yields the largest gains at longer horizons, consistent with its role in multi-step dynamics. Identity and inverse constraints also contribute by reducing stationary drift and residual motion after cancellation.

Additional ablations on rollout span, loss weight, and free-running versus teacher-forced training are provided in Appendix~\ref{app:ablation_more}.


\begin{figure*}[!ht]
    \centering
    \includegraphics[width=\textwidth]{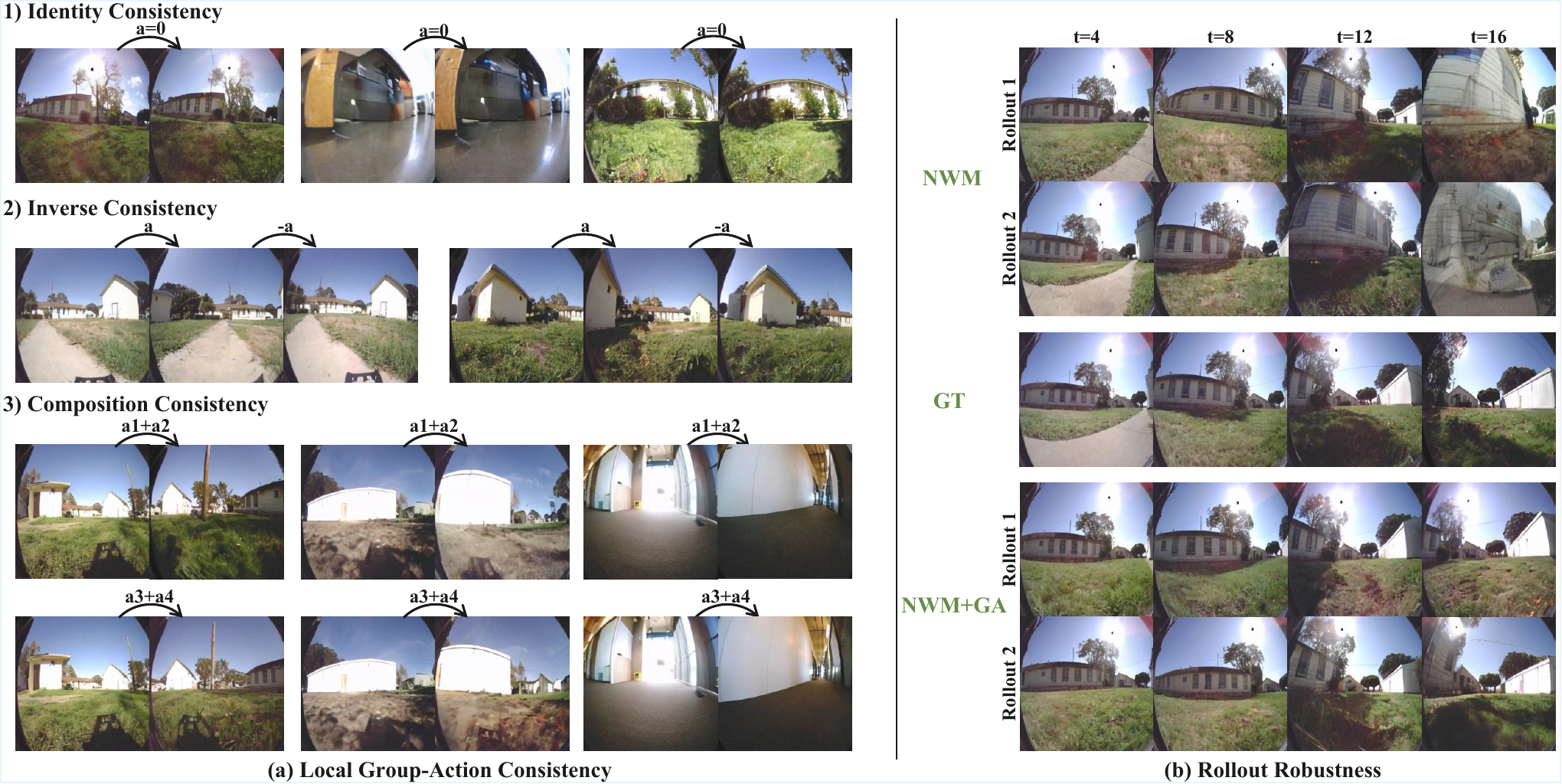}
    \vspace{-0.4cm}
    \caption{
    Qualitative results for Group-Action Consistency (GAC) and Group-Action Robustness (GAR).
    (a) Representative examples of identity, inverse, and compatibility probes, where GA regularization improves local action consistency.
    (b) Repeated stochastic rollouts from the same initial observation and identical action sequence.
    Although visual appearance varies across samples, the GA-regularized model exhibits more consistent motion evolution.
    }
    \label{fig:qualitative_main}
    \vspace{-0.4cm}
\end{figure*}

\subsection{Qualitative Results}
\label{sec:exp_qual}

Fig.~\ref{fig:qualitative_main} presents qualitative results for both evaluation metrics. For GAC, the GA-regularized model better preserves stationary observations under zero action, more accurately recovers after forward--inverse segments, and produces more consistent outcomes under locally equivalent decompositions. For GAR, although samples generated from the same initial observation and action sequence may still vary in appearance, their motion evolution remains more consistent. Overall, these results corroborate the quantitative findings, showing that the proposed constraints improve action faithfulness while maintaining visual realism.
Additional qualitative examples are provided in Appendix~\ref{app:qualitative_appendix}.
A complementary image-quality evaluation in Appendix~\ref{app:image_quality} further shows that these consistency gains do not come at the cost of perceptual generation quality.


\section{Conclusion}
In this work, we argue that consistency with the compositional structure of actions provides a principled criterion for assessing action faithfulness in video world models beyond visual realism. By formalizing world modeling as a group action, we provide a principled framework to diagnose and understand failures of existing video world models. We further introduce a unified method to enforce these structural constraints in latent space, along with metrics that evaluate both group action consistency and rollout stability. Extensive experimental results show that our approach consistently improves action-consistent dynamics while preserving visual quality.

\newpage
\bibliographystyle{unsrt}
\bibliography{reference}

\appendix


\section*{Appendix}

\paragraph{Appendix organization.}
This appendix provides supplementary material for the theoretical formulation, implementation details, evaluation protocol, additional results, ablations, qualitative examples, related work, and broader discussion.
Appendix~\ref{sec:appendix_group_action} gives the proof of the group-action formulation and clarifies its connection to $SE(2)$ navigation.
Appendix~\ref{sec:appendix_ga_training} details the latent rollout objective, constraint synthesis, action-space approximation, and the relation between latent training and recovered-state evaluation.
Appendix~\ref{app:exp_details} describes datasets, model variants, training setup, pose recovery, and metric protocols.
Appendices~\ref{app:additional_results}--\ref{app:qualitative_appendix} provide additional quantitative results, ablations, and qualitative examples.
Appendix~\ref{sec:appendix_relatedwork} reviews related work, and Appendix~\ref{app:discussion} discusses scope, broader implications, and limitations.

\section{Additional Details of the Group-Action Formulation}
\label{sec:appendix_group_action}

This appendix provides additional details for the group-action formulation in Sec.~\ref{sec:ga}.
We first prove Theorem~\ref{th:group_action}, then discuss why $SE(2)$ is the natural group for planar navigation, and finally clarify how the exact group-action ideal is operationalized through local observable conditions in video world models.

\subsection{Proof of Theorem~\ref{th:group_action}}
\label{app:proof_group_action}

\begin{proof}
We use the action-order convention adopted in Sec.~\ref{sec:ga}, where applying $a_1$ followed by $a_2$ corresponds to the composed action $a_1\circ a_2$.
Under this convention, the compatibility condition
\begin{equation}
    f_\theta(f_\theta(s,a_1),a_2)
    =
    f_\theta(s,a_1\circ a_2)
\end{equation}
and the identity condition
\begin{equation}
    f_\theta(s,e)=s
\end{equation}
are exactly the axioms of a right group action of $G$ on $\mathcal{S}$.
Therefore, $f_\theta$ defines a group action under this convention.

It remains to show the finite-sequence property.
For $n=1$, the statement is trivial.
Assume that for some $n-1$,
\begin{equation}
    f_\theta(\cdots f_\theta(f_\theta(s,a_1),a_2)\cdots,a_{n-1})
    =
    f_\theta(s,a_1\circ\cdots\circ a_{n-1}).
\end{equation}
Applying $a_n$ and using compatibility gives
\begin{align}
    &f_\theta(
        f_\theta(s,a_1\circ\cdots\circ a_{n-1}),
        a_n
    )
    \nonumber\\
    &\qquad =
    f_\theta(
        s,
        (a_1\circ\cdots\circ a_{n-1})\circ a_n
    )
    \nonumber\\
    &\qquad =
    f_\theta(
        s,
        a_1\circ\cdots\circ a_n
    ),
\end{align}
where the last equality follows from associativity of the group operation.
Thus, by induction,
\begin{equation}
    f_\theta(\cdots f_\theta(f_\theta(s,a_1),a_2)\cdots,a_n)
    =
    f_\theta(s,a_1\circ\cdots\circ a_n).
\end{equation}

The inverse condition follows from the group-action axioms.
Indeed, for any $a\in G$, compatibility and identity imply
\begin{equation}
    f_\theta(f_\theta(s,a),a^{-1})
    =
    f_\theta(s,a\circ a^{-1})
    =
    f_\theta(s,e)
    =
    s.
\end{equation}
We state inverse consistency explicitly in the main text because it is one of the most direct observable consequences of action cancellation in navigation rollouts.
\end{proof}

\subsection{Why $SE(2)$ for Navigation}
\label{app:why_se2}

In embodied navigation, actions typically represent local ego-motion increments in the plane.
An idealized planar motion can be written as an element of the special Euclidean group $SE(2)$:
\begin{equation}
    g
    =
    (R,\mathbf{p}),
    \qquad
    R\in SO(2),\quad
    \mathbf{p}\in\mathbb{R}^2,
\end{equation}
or equivalently as the homogeneous transformation
\begin{equation}
    G
    =
    \begin{bmatrix}
        R & \mathbf{p}\\
        \mathbf{0}^{\top} & 1
    \end{bmatrix}.
\end{equation}
The group operation is rigid-motion composition:
\begin{equation}
    (R_1,\mathbf{p}_1)\circ(R_2,\mathbf{p}_2)
    =
    (R_1R_2,\mathbf{p}_1+R_1\mathbf{p}_2).
\end{equation}
The identity element is
\begin{equation}
    e=(I,\mathbf{0}),
\end{equation}
and the inverse of a motion is
\begin{equation}
    (R,\mathbf{p})^{-1}
    =
    (R^\top,-R^\top\mathbf{p}).
\end{equation}

This group structure captures the algebraic relations that navigation actions should satisfy.
A zero motion should preserve the state, a motion followed by its inverse should cancel, and a sequence of motions should have an effect determined by their composition.
The group-action formulation in Sec.~\ref{sec:ga} uses this structure as the ideal state-space criterion for action-faithful world modeling.

In practical video world models, actions are often represented by normalized local ego-motion increments such as
\begin{equation}
    \mathbf{a}_t = (\Delta x_t,\Delta y_t,\Delta \theta_t).
\end{equation}
The main text uses this ego-motion action parameterization to instantiate observable local consequences of the $SE(2)$ group-action ideal.
The exact action-space implementation used for training is described separately in Appendix~\ref{sec:appendix_ga_training}.

\subsection{From Exact Group Action to Observable Local Conditions}
\label{app:observable_conditions}

The group-action formulation defines an ideal state-space property.
If the state $\mathbf{s}_t$ were directly observable and the transition map were deterministic, one could test whether
\begin{equation}
\begin{aligned}
    f_\theta(f_\theta(s,a_1),a_2)
    &=
    f_\theta(s,a_1\circ a_2),
    \\
    f_\theta(s,e)
    &=s,
    \\
    f_\theta(f_\theta(s,a),a^{-1})
    &=s,
\end{aligned}
\end{equation}
hold exactly.
In stochastic video rollout, however, the underlying state is not directly observed.
The model produces visual observations, and the induced motion must be recovered through an estimator $\Phi$.
Exact equality in the abstract state space is therefore not directly available as an evaluation target.

The main paper consequently focuses on local observable consequences of the group-action ideal.
The identity condition becomes zero-action preservation: a rollout under the neutral action should not introduce spurious motion.
The inverse condition becomes recovery after cancellation: executing a motion and then its inverse should return to the starting state.
The compatibility condition becomes local composition consistency: different short action decompositions with the same intended composed motion should induce compatible endpoints.

These conditions are not introduced as task-specific heuristics.
They are operational tests of whether the model-induced rollout preserves the basic algebraic consequences of planar motion composition.
Their violations reveal failure modes that can remain hidden when evaluating only visual plausibility.

\subsection{Scope of the Local Composition Approximation}
\label{app:local_composition_scope}

The group-action formulation is grounded in the SE(2) structure of planar rigid motions.
Exact SE(2) composition is generally non-commutative:
the effect of a translation followed by a rotation is not identical to the effect of the same rotation followed by the same translation.
Therefore, the composition condition used in this paper should not be interpreted as assuming that SE(2) composition is commutative, nor as claiming that arbitrary trajectories with the same cumulative displacement are equivalent.

Our use of composition consistency is a local operational approximation for short ego-motion windows.
In the datasets and training setup considered here, actions are represented as normalized local ego-motion increments.
For short windows with small rotations, redistributing the same intended local displacement provides a practical probe for whether the model is overly sensitive to the temporal realization of a local motion segment.
The goal is not to replace the exact SE(2) operation with a globally valid additive rule.
Rather, the approximation asks whether two short action decompositions with matched intended motion induce compatible endpoint states.

This distinction is important for interpreting both the metric and the training objective.
The composition probe should be read as a local consistency test for action-conditioned rollout, not as a statement that all trajectories with equal cumulative displacement should coincide.
Accordingly, all composition probes are restricted to short windows, and the window lengths used in evaluation are chosen to remain within this local regime.

\section{Additional Details on Group-Action Training}
\label{sec:appendix_ga_training}

This appendix provides implementation details for the group-action training objective introduced in Sec.~\ref{sec:method}.
We describe the free-running latent rollout setting, the synthesis of training constraints from existing trajectories, the action-space approximations used during training, the practical per-batch objective, and the relation between latent training and recovered-state evaluation.
Evaluation protocols and metric aggregation are described separately in Appendix~\ref{app:exp_details}.

\subsection{Free-Running Latent Rollout}
\label{app:free_running_latent_rollout}

The group-action objective is computed on latent endpoints generated by the model's own rollout process.
Starting from a sampled latent state $\mathbf{z}_t$, the model is rolled out under a synthesized action segment without injecting ground-truth observations inside the constrained interval.
This setting matches the object being regularized, namely the action-conditioned transition induced by the model itself.

There are two possible rollout paradigms for applying such constraints.
Teacher-forced rollout supplies ground-truth visual context at intermediate steps, while free-running rollout lets the model condition on its own generated latent states.
Since the identity, inverse, and composition conditions characterize how the learned transition behaves under self-composition, we adopt free-running latent rollout inside the constrained interval.
This allows the group-action objective to act directly on the internal dynamics used during generation and exposes violations that teacher forcing may hide.

For memory efficiency and training stability, gradients are propagated only through the sampled local interval.
Latent states before the interval are detached from the computation graph.
This truncated rollout is used as an optimization device and does not change the definition of the group-action constraints.
We examine the effect of this free-running design in the ablation study.

\subsection{Training Constraint Synthesis and Per-Batch Sampling}
\label{app:ga_constraint_synthesis_details}

For each training trajectory, a starting index $t$ is sampled from valid temporal positions.
The local rollout horizon is sampled as
\begin{equation}
    l \sim \mathcal{U}\{1,\ldots,L\},
\end{equation}
where $L$ controls the maximum span over which a group-action constraint is imposed.
Short horizons emphasize immediate transition behavior, while longer horizons expose errors that appear only after repeated transition composition.
Unless otherwise specified, $L=4$ is used in our experiments.

Given the native local action segment $\mathbf{u}_{1:l}$ defined in Sec.~\ref{sec:ga_constraint_synthesis}, the identity, inverse, and composition constraints can in principle all be evaluated for the same sampled segment.
In practice, computing all three constraints in every batch increases memory and computation because each constraint requires a separate latent rollout.
We therefore sample one constraint type per training batch.
The active group-action loss is one of
\begin{equation}
    \mathcal{L}_{\mathrm{GA}}^{(b)}
    \in
    \{
    \mathcal{L}_{\mathrm{id}},
    \mathcal{L}_{\mathrm{inv}},
    \mathcal{L}_{\mathrm{comp}}
    \}.
\end{equation}
This corresponds to a stochastic realization of the full objective in Eq.~\ref{eq:training_objective}: in each batch, only the sampled group-action term is activated, while the other two terms have zero effective weight.
Equivalently, the per-batch objective can be written as
\begin{equation}
    \mathcal{L}
    =
    \mathcal{L}_{\mathrm{diff}}
    +
    \lambda \mathcal{L}_{\mathrm{GA}}^{(b)}.
\end{equation}
Across training, this stochastic sampling exposes the model to all three group-action constraints while keeping the rollout-based objective tractable.

\subsection{Action-Space Approximation for Training}
\label{app:action_inverse_composition}

The theoretical formulation in Sec.~\ref{sec:ga} is defined in terms of the group operation of planar rigid motions.
In the video world models considered in this work, however, actions are provided as normalized local ego-motion increments.
The training constraints are therefore instantiated under this action parameterization, following the operational form used in Sec.~\ref{sec:instantiation}.

For the inverse constraint, let
\begin{equation}
    \mathbf{u}_{1:l}
    =
    (\mathbf{a}_{t},\ldots,\mathbf{a}_{t+l-1})
\end{equation}
be a local action segment.
The reverse part is constructed by reversing the temporal order and negating each normalized increment:
\begin{equation}
    \hat{\mathbf{u}}_{1:l}
    =
    (-\mathbf{a}_{t+l-1},\ldots,-\mathbf{a}_{t}).
\end{equation}
The full inverse training segment is then the forward--inverse cycle
\begin{equation}
    \mathbf{u}^{\mathrm{inv}}_{1:2l}
    =
    (\mathbf{u}_{1:l}, \hat{\mathbf{u}}_{1:l})
    =
    (\mathbf{a}_{t},\ldots,\mathbf{a}_{t+l-1},
    -\mathbf{a}_{t+l-1},\ldots,-\mathbf{a}_{t}).
\end{equation}
This implements the cancellation relation used by the inverse constraint under the normalized ego-motion parameterization.

For the composition constraint, exact equality under $SE(2)$ composition would require composing rigid motions in the full group representation.
Instead, consistent with the main text, we construct local composition probes by accumulating normalized action increments over short ego-motion windows.
Given a segment $\mathbf{u}^{A}_{1:l}$, the alternative segment $\mathbf{u}^{B}_{1:l}$ is constructed as in Eq.~\ref{eq:alternative_decomposition}, and satisfies
\begin{equation}
    \sum_{i=1}^{l} \mathbf{u}^{A}_{i}
    =
    \sum_{i=1}^{l} \mathbf{u}^{B}_{i}.
\end{equation}
Thus, the two segments preserve the same accumulated increments while changing the temporal allocation of the intended local motion.

This approximation is used only for local constraint synthesis.
It should not be interpreted as replacing exact $SE(2)$ composition with a globally valid additive rule.
Rather, it provides an efficient way to construct alternative short-horizon decompositions from existing trajectories, enabling the model to be regularized against sensitivity to different temporal realizations of the same intended local motion.

The Dirichlet redistribution in Eq.~\ref{eq:alternative_decomposition} controls the diversity of these alternative decompositions.
The concentration parameter $\alpha$ determines how evenly the accumulated increments are distributed across the segment.
Smaller values produce more uneven allocations, while larger values produce smoother redistributions.

\subsection{Latent Training versus State-Space Evaluation}
\label{app:latent_vs_state}

Latent training and recovered-state evaluation serve different roles.
The group-action objective is imposed in latent space because this is where the video world model composes its internal rollout dynamics before decoding future observations.
Applying the same objective directly in recovered state space would require decoding generated frames and running an external pose estimator inside the training loop, which is computationally expensive and unstable for gradient-based optimization.

Recovered-state evaluation is used because action faithfulness is ultimately a property of induced motion.
Recovered trajectories provide an interpretable representation for measuring identity drift, inverse recovery error, composition mismatch, and repeated-rollout dispersion.

The method does not assume that latent distance is identical to recovered-state distance.
The latent losses and the recovered-state metrics are both derived from the same group-action conditions.
The former provide a practical optimization signal, while the latter measure whether generated rollouts exhibit the expected motion-level consequences.

\section{Additional Experimental Details}
\label{app:exp_details}

This appendix provides additional details for the experiments in Sec.~\ref{sec:exp}.
We describe the datasets, model variants, training setup, pose recovery, GAC evaluation protocol, GAC probe design and aggregation, and GAR rollout protocol.

\subsection{Datasets}
\label{app:datasets}

Experiments are conducted on RECON~\cite{shah2021recon}, SCAND~\cite{karnan2022scand}, and HuRoN~\cite{hirose2023sacson}, following the preprocessing and data split protocol of NWM~\cite{bar2025nwm}.
Relative actions are derived from ground-truth pose differences and normalized so that action scales remain comparable across datasets.

\paragraph{RECON.}
RECON is an outdoor robotics dataset collected using a Clearpath Jackal UGV platform.
It contains approximately 40 hours of data across 9 open-world environments.
Following NWM~\cite{bar2025nwm}, 9,468 video segments are used for training and 2,367 video segments for testing.
Its emphasis on long-range traversal in visually diverse open-world environments makes it particularly suitable for evaluating whether rollout instability accumulates over extended action sequences.

\paragraph{SCAND.}
SCAND is a robotics dataset of socially compliant navigation demonstrations collected using a wheeled Clearpath Jackal and a legged Boston Dynamics Spot in both indoor and outdoor settings at UT Austin.
The full dataset contains 8.7 hours of data, 138 trajectories, and 25 miles of demonstrations.
Following NWM~\cite{bar2025nwm}, 484 video segments are used for training and 121 video segments for testing.
Compared with RECON, SCAND contains richer local interaction structure and denser socially conditioned motion variation, making it useful for evaluating whether local group-action consistency is preserved in socially complex settings.

\paragraph{HuRoN.}
HuRoN is a robotics dataset of social navigation interactions collected using a Robot Roomba in indoor environments at UC Berkeley.
The dataset contains more than 75 hours of data across 5 environments and includes over 4,000 human interactions.
Following NWM~\cite{bar2025nwm}, 2,451 video segments are used for training and 613 video segments for testing.
Its emphasis on socially unobtrusive navigation makes it relevant for evaluating whether action-conditioned rollout remains stable in interaction-heavy and behaviorally sensitive environments.

\subsection{Model Variants and Adaptation}
\label{app:model_adaptation}

NWM~\cite{bar2025nwm} is used as the main testbed throughout the experiments.
It is a representative diffusion-based navigation world model with stochastic action-conditioned rollout, making it suitable for studying whether explicit group-action constraints improve rollout stability.
Following its default setting, a CDiT-XL backbone with approximately 1B parameters is used together with a 4-frame context window.
Visual observations are tokenized using the Stable Diffusion VAE~\cite{blattmann2023svd}, following the design adopted in NWM.

DIAMOND~\cite{alonso2024diamond} is included as an additional action-conditioned world-model baseline.
Since DIAMOND was originally introduced in a different domain, we adapt its public implementation to the embodied navigation setting using a training setup similar in spirit to NWM.
This additional baseline is included to test whether the proposed GA framework generalizes beyond a single world-model architecture.

\subsection{Training Setup}
\label{app:training_setup}

Optimization is performed with AdamW~\cite{adamw}.
For GA training, full fine-tuning is used and the initial learning rate is set to $1\times10^{-6}$.
Training is conducted on 8 A800 GPUs with a per-GPU batch size of 1.
Unless otherwise specified, the remaining optimization settings follow the corresponding base model as closely as possible.
The group-action objective is applied only during fine-tuning.
The maximum local rollout horizon is set to $L=4$ by default.
The group-action loss weight $\lambda$ is selected from $\{0.1,0.5,1.0\}$, and $\lambda=0.5$ is used by default.
As described in Appendix~\ref{app:ga_constraint_synthesis_details}, one group-action constraint type is sampled in each training batch for efficiency.

\subsection{Pose Recovery and Trajectory Alignment}
\label{app:pose_recovery}

Generated videos are mapped to recovered state trajectories using the pose estimator $\Phi$ defined in Sec.~\ref{sec:instantiation}, implemented with the off-the-shelf DROID-SLAM~\cite{teed2021droid} toolkit.
All motion-level errors are computed using the state distance $d(\cdot,\cdot)$ in Eq.~\ref{eq:state_distance}.

For GAR evaluation, both aligned and non-aligned trajectory errors are reported.
Aligned errors remove global pose drift before computing trajectory discrepancy, while non-aligned errors measure raw deviation in the recovered trajectory.
The aligned setting emphasizes local trajectory shape, whereas the non-aligned setting additionally reflects global drift in position and heading.

\subsection{GAC Evaluation Protocol}
\label{app:gac_eval_protocol}

This subsection specifies the evaluation configurations used to compute the Group-Action Consistency (GAC) error in Sec.~\ref{sec:ga_metrics}.
Generated videos are first mapped to recovered state trajectories using $\Phi$.
The main text defines $\Delta_{\mathrm{id}}$, $\Delta_{\mathrm{inv}}$, and $\Delta_{\mathrm{comp}}$ as the averaged errors for identity, inverse, and local composition consistency.
Here we spell out the corresponding probe families.

\paragraph{Identity evaluation.}
Identity consistency is evaluated by inserting zero-action segments into the action stream.
Let $\mathcal{T}_{\mathrm{id}}^{(k,l)}$ denote the set of evaluated zero-action segments, where $k$ is the number of inserted segments and $l$ is the length of each segment.
For a segment starting at timestep $t$, the identity drift is measured by comparing the recovered states before and after the zero-action interval:
\begin{equation}
\Delta_{\mathrm{id}}^{(k,l)}
=
\frac{1}{|\mathcal{T}_{\mathrm{id}}^{(k,l)}|}
\sum_{t \in \mathcal{T}_{\mathrm{id}}^{(k,l)}}
d\bigl(\mathbf{s}_{t+l}^{\mathrm{id}}, \mathbf{s}_t\bigr).
\end{equation}
A larger value indicates stronger drift under the neutral action.

\paragraph{Inverse evaluation.}
Inverse consistency is evaluated by executing a local action segment followed by its inverse sequence.
Let $\mathcal{T}_{\mathrm{inv}}^{(k,l)}$ denote the set of evaluated forward--inverse segments, where $k$ is the number of evaluated segments and $l$ controls the temporal scale of each segment.
The inverse recovery error is
\begin{equation}
\Delta_{\mathrm{inv}}^{(k,l)}
=
\frac{1}{|\mathcal{T}_{\mathrm{inv}}^{(k,l)}|}
\sum_{t \in \mathcal{T}_{\mathrm{inv}}^{(k,l)}}
d\bigl(\mathbf{s}_{t,\mathrm{rec}}^{\mathrm{inv}}, \mathbf{s}_t\bigr),
\end{equation}
where $\mathbf{s}_{t,\mathrm{rec}}^{\mathrm{inv}}$ is the recovered state after executing the action segment and its inverse.
A larger value indicates weaker inverse consistency.

\paragraph{Composition evaluation.}
Local composition consistency is evaluated through alternative decompositions of the same local motion.
Given a local action segment $\mathbf{a}_{1:l}$, an alternative segment $\mathbf{b}_{1:l}$ is constructed by redistributing the same cumulative displacement:
\begin{equation}
\mathbf{b}_i
=
w_i \sum_{j=1}^{l} \mathbf{a}_j,
\qquad
\mathbf{w}\sim\mathrm{Dir}(\alpha),
\end{equation}
which satisfies
\begin{equation}
\sum_{i=1}^{l}\mathbf{b}_i
=
\sum_{i=1}^{l}\mathbf{a}_i.
\end{equation}
This construction changes the temporal allocation of the local motion while preserving its cumulative displacement.
The model is rolled out under the original segment and the recomposed segment.
Let $\mathcal{T}_{\mathrm{comp}}^{(l)}$ denote the set of evaluated composition segments.
The composition mismatch is
\begin{equation}
\Delta_{\mathrm{comp}}^{(l)}
=
\frac{1}{|\mathcal{T}_{\mathrm{comp}}^{(l)}|}
\sum_{t \in \mathcal{T}_{\mathrm{comp}}^{(l)}}
d\bigl(\mathbf{s}_{t+l}^{A}, \mathbf{s}_{t+l}^{B}\bigr),
\end{equation}
where $\mathbf{s}_{t+l}^{A}$ and $\mathbf{s}_{t+l}^{B}$ are the recovered endpoint states under the original and recomposed segments.
A larger value indicates weaker local composition consistency.

\subsection{GAC Probe Design and Metric Aggregation}
\label{app:gac_probe_design}

The GAC metric is implemented as a fixed controlled probe suite.
All probe configurations are specified before evaluation and shared across methods, so the reported differences are not affected by method-specific probe selection.
The probe suite evaluates local observable consequences of the group-action ideal at multiple difficulty levels.

For identity and inverse consistency, each configuration is indexed by $(k,l)$.
The parameter $k$ controls the number of inserted or evaluated local segments.
It tests whether small violations accumulate when the same local group-action relation is applied repeatedly.
The parameter $l$ controls the temporal length of each segment.
It tests whether the violation grows as the local rollout span becomes longer.
In the main analysis, one factor is varied while the other is fixed, using $k,l\in\{1,3,5\}$.
This produces small, medium, and large local probes while keeping the evaluation interpretable.
We do not use a dense grid over all $(k,l)$ pairs because it would mix the two sources of difficulty and introduce redundant configurations.

For local composition consistency, the probe compares two decompositions of the same local motion window.
The natural scale parameter is therefore the window length $l$.
We use $l\in\{2,4,6\}$ to cover increasing local composition difficulty while staying within the short-horizon regime where the local composition approximation is intended to apply.
The goal is not to claim that long-horizon trajectories with the same cumulative displacement are equivalent.
Rather, the probe tests whether the model is overly sensitive to different temporal realizations of a short local motion.

For each group-action condition, configuration-level errors are first averaged over fixed valid starting positions and evaluation sequences.
The component errors are then averaged over the corresponding evaluation family:
\begin{equation}
\bar{\Delta}_{\mathrm{id}}
=
\frac{1}{|\mathcal{K}_{\mathrm{id}}|}
\sum_{(k,l)\in\mathcal{K}_{\mathrm{id}}}
\Delta_{\mathrm{id}}^{(k,l)},
\qquad
\bar{\Delta}_{\mathrm{inv}}
=
\frac{1}{|\mathcal{K}_{\mathrm{inv}}|}
\sum_{(k,l)\in\mathcal{K}_{\mathrm{inv}}}
\Delta_{\mathrm{inv}}^{(k,l)}.
\end{equation}
For local composition consistency, we average over the evaluated window lengths:
\begin{equation}
\bar{\Delta}_{\mathrm{comp}}
=
\frac{1}{|\mathcal{L}_{\mathrm{comp}}|}
\sum_{l\in\mathcal{L}_{\mathrm{comp}}}
\Delta_{\mathrm{comp}}^{(l)}.
\end{equation}
The aggregate GAC error reported in the experiments is
\begin{equation}
    \mathcal{E}_{\mathrm{GAC}}
    =
    \frac{1}{3}
    \left(
    \bar{\Delta}_{\mathrm{id}}
    +
    \bar{\Delta}_{\mathrm{inv}}
    +
    \bar{\Delta}_{\mathrm{comp}}
    \right).
\end{equation}
This aggregation is well-defined because all components are measured using the same recovered-state distance $d(\cdot,\cdot)$.
Equal weighting avoids introducing a task-specific preference among identity preservation, inverse cancellation, and local composition consistency.
We report both the aggregate score and the component errors, so changes in $\mathcal{E}_{\mathrm{GAC}}$ can be traced to the corresponding local group-action conditions.

\subsection{GAR Evaluation Protocol}
\label{app:gar_eval_protocol}

Group-Action Robustness (GAR) is evaluated under repeated stochastic rollouts conditioned on the same initial observation and identical action sequence.
Unless otherwise specified, each model is sampled $N=5$ times.
The corresponding generated videos are mapped to recovered state trajectories
$\{\mathbf{s}^{(i)}_{1:T}\}_{i=1}^{N}$.
The GAR error is computed using Eq.~\ref{eq:gar_error}.
Mean and standard deviation are reported across evaluation sequences.
Both 16-frame and 64-frame horizons are evaluated to measure short-range and longer-range accumulation of rollout errors.

\subsection{Assets and Licenses}
\label{app:assets_licenses}

This work uses existing research datasets and model implementations, including RECON~\cite{shah2021recon}, SCAND~\cite{karnan2022scand}, HuRoN~\cite{hirose2023sacson}, NWM~\cite{bar2025nwm}, DIAMOND~\cite{alonso2024diamond}, and DROID-SLAM~\cite{teed2021droid}.
We credit the original creators through citation and use these assets for research purposes following their released terms and licenses.
No new dataset is collected or redistributed in this work.

\section{Additional Quantitative Results}
\label{app:additional_results}

This appendix provides supplementary quantitative results beyond the main RECON analysis in Sec.~\ref{sec:exp}.

\subsection{Additional GAR Results}
\label{app:gar_more}
Table~\ref{tab:gar_more} reports additional GAR results on SCAND and HuRoN under repeated stochastic rollouts with identical initial observations and action sequences.
The same evaluation protocol as in Sec.~\ref{sec:exp_gar} is utilized on both datasets.
On both SCAND and HuRoN, GA training reduces rollout dispersion for both DIAMOND and NWM, showing that the improvement in rollout-level robustness is not limited to the main RECON setting.
The gains are consistently observed under both aligned and non-aligned evaluation, indicating improved local trajectory agreement as well as reduced global drift.

\begin{table*}[!ht]
\caption{
Additional Group-Action Robustness (GAR) results on SCAND and HuRoN under repeated stochastic rollouts with identical action sequences.
Lower is better.
``Aligned'' removes global pose drift before computing trajectory error, while ``Non-aligned'' measures raw trajectory deviation.
}
\label{tab:gar_more}
\centering
\setlength{\tabcolsep}{6.5pt}
\renewcommand{\arraystretch}{1.08}
\begin{tabular}{llcccc}
\toprule
\multirow{2}{*}{Dataset} & \multirow{2}{*}{Method} & \multicolumn{2}{c}{16 Frames} & \multicolumn{2}{c}{64 Frames} \\
\cmidrule(lr){3-4}\cmidrule(lr){5-6}
& & Aligned $\downarrow$ & Non-aligned $\downarrow$ & Aligned $\downarrow$ & Non-aligned $\downarrow$ \\
\midrule
\multirow{4}{*}{SCAND}
& DIAMOND~\cite{alonso2024diamond}
& $0.63 \pm 0.42$
& $1.24 \pm 0.86$
& $2.06 \pm 1.31$
& $6.42 \pm 4.16$ \\
& DIAMOND + GA
& $0.53 \pm 0.35$
& $1.03 \pm 0.72$
& $1.68 \pm 1.05$
& $5.28 \pm 3.36$ \\
& NWM~\cite{bar2025nwm}
& $0.46 \pm 0.28$ 
& $0.91 \pm 0.64$ 
& $1.45 \pm 0.85$ 
& $4.08 \pm 2.69$ \\
& NWM + GA
& $\mathbf{0.37 \pm 0.23}$
& $\mathbf{0.72 \pm 0.50}$
& $\mathbf{1.12 \pm 0.68}$
& $\mathbf{3.39 \pm 2.04}$ \\
\midrule
\multirow{4}{*}{HuRoN}
& DIAMOND~\cite{alonso2024diamond}
& $0.65 \pm 0.23$
& $1.37 \pm 0.56$
& $2.34 \pm 0.92$
& $7.36 \pm 3.20$ \\
& DIAMOND + GA
& $0.55 \pm 0.19$
& $1.13 \pm 0.46$
& $1.91 \pm 0.75$
& $6.04 \pm 2.58$ \\
& NWM~\cite{bar2025nwm}
& $0.47 \pm 0.16$
& $0.98 \pm 0.35$
& $1.70 \pm 0.58$
& $5.13 \pm 2.22$ \\
& NWM + GA
& $\mathbf{0.39 \pm 0.13}$
& $\mathbf{0.77 \pm 0.29}$
& $\mathbf{1.33 \pm 0.47}$
& $\mathbf{4.15 \pm 1.70}$ \\
\bottomrule
\end{tabular}
\end{table*}

\subsection{Additional GAC Results}
\label{app:gac_more}

Table~\ref{tab:gac_aggregate_more} reports aggregated GAC errors on SCAND and HuRoN.
Component errors are averaged over the evaluated probe configurations, and $\mathcal{E}_{\mathrm{GAC}}$ is computed as the mean of identity, inverse and composition errors.
On both datasets, GA training reduces the aggregate GAC error for both DIAMOND and NWM.
This supports the conclusion that the latent group-action objective improves local action faithfulness beyond the main RECON dataset and beyond a single backbone architecture.

\begin{table*}[!ht]
\caption{
Aggregated Group-Action Consistency (GAC) error on SCAND and HuRoN.
Component errors are averaged over evaluated probe configurations.
The last column reports $\mathcal{E}_{\mathrm{GAC}}$ as the mean of the three component errors.
Standard deviations are shown for component metrics.
Lower is better.
}
\label{tab:gac_aggregate_more}
\centering
\setlength{\tabcolsep}{6.5pt}
\renewcommand{\arraystretch}{1.08}
\begin{tabular}{llcccc}
\toprule
Dataset
& Method
& $\bar{\Delta}_{\mathrm{id}} \downarrow$
& $\bar{\Delta}_{\mathrm{inv}} \downarrow$
& $\bar{\Delta}_{\mathrm{comp}} \downarrow$
& $\mathcal{E}_{\mathrm{GAC}} \downarrow$ \\
\midrule

\multirow{4}{*}{SCAND}
& DIAMOND~\cite{alonso2024diamond}
& $6.12 \pm 2.95$
& $5.92 \pm 2.81$
& $3.56 \pm 2.12$
& $5.20$ \\
& DIAMOND + GA
& $5.56 \pm 2.71$
& $5.43 \pm 2.53$
& $2.96 \pm 1.82$
& $4.65$ \\
& NWM~\cite{bar2025nwm}
& $5.38 \pm 2.63$
& $5.51 \pm 2.55$
& $2.68 \pm 1.61$
& $4.52$ \\
& NWM + GA
& $\mathbf{4.84 \pm 2.32}$
& $\mathbf{4.67 \pm 2.16}$
& $\mathbf{2.06 \pm 1.31}$
& $\mathbf{3.86}$ \\

\midrule

\multirow{4}{*}{HuRoN}
& DIAMOND~\cite{alonso2024diamond}
& $7.08 \pm 2.86$
& $6.77 \pm 2.91$
& $3.31 \pm 1.84$
& $5.72$ \\
& DIAMOND + GA
& $6.36 \pm 2.66$
& $6.08 \pm 2.55$
& $2.79 \pm 1.50$
& $5.08$ \\
& NWM~\cite{bar2025nwm}
& $6.07 \pm 2.48$
& $6.39 \pm 2.64$
& $2.53 \pm 1.37$
& $5.00$ \\
& NWM + GA
& $\mathbf{5.36 \pm 2.25}$
& $\mathbf{5.49 \pm 2.32}$
& $\mathbf{1.99 \pm 1.15}$
& $\mathbf{4.28}$ \\

\bottomrule
\end{tabular}
\end{table*}

\subsection{Probe-wise Analysis of GAC Errors}
\label{app:gac_probe_analysis}

Fig.~\ref{fig:asc_trends} provides a probe-wise view of GAC errors on RECON.
The probe parameters are not intended to define a strictly monotonic difficulty scale.
Instead, they expose complementary failure modes of group-action consistency, including repeated local relations, longer temporal extent, and alternative decompositions of the same local motion.

For inverse consistency, errors tend to increase as the evaluated segment becomes longer or is repeated more times.
This is expected because the model must compose a forward motion with its inverse and recover the starting state.
Residual transition errors accumulated during the forward segment can therefore remain after cancellation.
GA regularization consistently lowers the inverse error, indicating improved action cancellation under composed rollouts.

For local composition consistency, longer windows lead to larger endpoint mismatch.
Longer windows involve more transition steps and give two equivalent local decompositions more opportunities to diverge.
The reduction after GA regularization shows that the model becomes less sensitive to the particular temporal decomposition of a local motion, which is the desired behavior under the group-action formulation.

Identity consistency follows a less monotonic pattern.
Repeated or longer zero-action segments can make the stationary condition more explicit, which may reduce mean drift in intermediate settings.
At larger settings, however, the probe also imposes a stricter preservation requirement, since the model must remain static across more pause locations or longer pause durations.
The increased error or variance in these cases suggests residual motion under extended zero-action rollouts, rather than exact realization of the identity transformation.
GA regularization reduces identity errors in most settings, showing improved but still imperfect preservation under zero action.

\subsection{Image-Quality Evaluation}
\label{app:image_quality}

Table~\ref{tab:image_quality_metrics} reports image-quality metrics for generated rollouts on RECON.
We evaluate LPIPS~\cite{LPIPS}, DreamSim~\cite{fu2023dreamsim}, and PSNR between generated frames and the corresponding ground-truth frames following NWM~\cite{bar2025nwm}, using the same preprocessing and sampling pipeline across methods.
This evaluation checks whether the gains in GAC and GAR come at the cost of visual prediction quality.
All methods are evaluated under the same preprocessing, sampling, and metric-computation pipeline.
The results show that GA regularization keeps image-quality metrics comparable to the base model.
This suggests that the proposed objective improves action-conditioned dynamics without noticeably degrading perceptual or pixel-level generation quality.

\begin{table}[!ht]
\caption{
Image-quality evaluation on RECON under the same evaluation pipeline.
LPIPS and DreamSim measure perceptual dissimilarity, while PSNR measures pixel-level reconstruction quality.
Lower is better for LPIPS and DreamSim, while higher is better for PSNR.
}
\label{tab:image_quality_metrics}
\centering
\setlength{\tabcolsep}{6pt}
\renewcommand{\arraystretch}{1.08}
\begin{tabular}{lccc}
\toprule
Method & LPIPS $\downarrow$ & DreamSim $\downarrow$ & PSNR $\uparrow$ \\
\midrule
NWM & $0.386 \pm 0.004$ & $0.141 \pm 0.002$ & $13.86 \pm 0.08$ \\
NWM + GA & $\mathbf{0.364 \pm 0.003}$ & $\mathbf{0.130 \pm 0.002}$ & $\mathbf{14.29 \pm 0.07}$ \\
\bottomrule
\end{tabular}
\end{table}

\section{Additional Ablation Studies}
\label{app:ablation_more}

This appendix provides supplementary ablation studies on RECON.
We analyze the effect of the maximum rollout span, the group-action loss weight, and the training rollout paradigm.

\subsection{Effect of Maximum Rollout Span}
\label{app:ablation_rollout_span}

Table~\ref{tab:ablation_rollout_span} studies the effect of the maximum rollout span $L$ used in group-action training.
Increasing the rollout span generally improves GAR, indicating that enforcing group-action constraints over multi-step transitions is important for reducing accumulated rollout error.
The performance becomes similar once $L$ reaches a moderate range, with only marginal differences between $L=4$ and $L=6$.
We therefore use $L=4$ as the default setting, which offers a good balance between performance and training cost.

\begin{table*}[!ht]
\caption{
Effect of the maximum rollout span $L$ used in group-action training on RECON.
Lower is better.
``Aligned'' removes global pose drift before computing trajectory error, while ``Non-aligned'' measures raw trajectory deviation.
}
\label{tab:ablation_rollout_span}
\centering
\setlength{\tabcolsep}{7pt}
\renewcommand{\arraystretch}{1.08}
\begin{tabular}{lcccc}
\toprule
\multirow{2}{*}{$L$} & \multicolumn{2}{c}{16 Frames} & \multicolumn{2}{c}{64 Frames} \\
\cmidrule(lr){2-3}\cmidrule(lr){4-5}
& Aligned $\downarrow$ & Non-aligned $\downarrow$ & Aligned $\downarrow$ & Non-aligned $\downarrow$ \\
\midrule
2
& $0.27 \pm 0.22$
& $0.53 \pm 0.37$
& $0.49 \pm 0.36$
& $1.35 \pm 0.89$ \\

4
& $\mathbf{0.26 \pm 0.20}$
& $\mathbf{0.49 \pm 0.35}$
& $0.46 \pm 0.33$
& $\mathbf{1.28 \pm 0.82}$ \\

6
& $0.26 \pm 0.21$
& $0.50 \pm 0.35$
& $\mathbf{0.45 \pm 0.33}$
& $1.29 \pm 0.83$ \\
\bottomrule
\end{tabular}
\end{table*}

\subsection{Effect of Group-Action Loss Weight}
\label{app:ablation_lambda}

Table~\ref{tab:ablation_lambda} studies the effect of the group-action loss weight $\lambda$.
Beyond hyperparameter sensitivity, this experiment tests whether the observed gain could be attributed simply to continued fine-tuning.
The $\lambda=0$ case corresponds to training under the same pipeline without group-action constraints.
Its performance is almost unchanged relative to the baseline, showing only marginal variation.
Introducing a nonzero group-action weight leads to clear improvements, with the best overall performance obtained at an intermediate setting.
When $\lambda$ becomes larger, the gain saturates slightly, suggesting that excessively strong group-action training may start to compete with the base diffusion objective.

\begin{table*}[!ht]
\caption{
Effect of the group-action loss weight $\lambda$ on RECON.
Lower is better.
``Aligned'' removes global pose drift before computing trajectory error, while ``Non-aligned'' measures raw trajectory deviation.
}
\label{tab:ablation_lambda}
\centering
\setlength{\tabcolsep}{7pt}
\renewcommand{\arraystretch}{1.08}
\begin{tabular}{lcccc}
\toprule
\multirow{2}{*}{$\lambda$} & \multicolumn{2}{c}{16 Frames} & \multicolumn{2}{c}{64 Frames} \\
\cmidrule(lr){2-3}\cmidrule(lr){4-5}
& Aligned $\downarrow$ & Non-aligned $\downarrow$ & Aligned $\downarrow$ & Non-aligned $\downarrow$ \\
\midrule
Baseline
& $0.32 \pm 0.26$
& $0.65 \pm 0.41$
& $0.62 \pm 0.44$
& $1.52 \pm 1.27$ \\

0
& $0.33 \pm 0.23$
& $0.61 \pm 0.42$
& $0.60 \pm 0.44$
& $1.47 \pm 1.21$ \\

0.1
& $0.27 \pm 0.21$
& $0.52 \pm 0.37$
& $0.48 \pm 0.35$
& $1.34 \pm 0.88$ \\

0.5
& $\mathbf{0.26 \pm 0.20}$
& $\mathbf{0.49 \pm 0.35}$
& $\mathbf{0.46 \pm 0.33}$
& $\mathbf{1.28 \pm 0.82}$ \\

1.0
& $0.27 \pm 0.20$
& $0.50 \pm 0.36$
& $0.47 \pm 0.34$
& $1.31 \pm 0.86$ \\
\bottomrule
\end{tabular}
\end{table*}

\subsection{Teacher-Forcing versus Free-Running Training}
\label{app:ablation_tf_vs_fr}

Table~\ref{tab:ablation_tf_vs_fr} compares two ways of applying group-action constraints during training.
In the teacher-forcing variant, the same constraints are constructed, but each transition is conditioned on ground-truth history rather than the model's own previously generated states.
In the free-running variant, the model rolls forward under its own generated latent states, and the group-action losses are imposed on the resulting rollout.
The teacher-forcing variant improves over the baseline, indicating that the group-action relations are meaningful even when applied under locally corrected context.
However, its gain remains consistently smaller than that of free-running training, especially at longer horizons.
This supports the free-running latent rollout design discussed in Appendix~\ref{app:free_running_latent_rollout}.

\begin{table*}[!ht]
\caption{
Teacher-forcing versus free-running group-action training on RECON.
Lower is better.
``Aligned'' removes global pose drift before computing trajectory error, while ``Non-aligned'' measures raw trajectory deviation.
}
\label{tab:ablation_tf_vs_fr}
\centering
\setlength{\tabcolsep}{7pt}
\renewcommand{\arraystretch}{1.08}
\begin{tabular}{lcccc}
\toprule
\multirow{2}{*}{Training scheme} & \multicolumn{2}{c}{16 Frames} & \multicolumn{2}{c}{64 Frames} \\
\cmidrule(lr){2-3}\cmidrule(lr){4-5}
& Aligned $\downarrow$ & Non-aligned $\downarrow$ & Aligned $\downarrow$ & Non-aligned $\downarrow$ \\
\midrule
Baseline
& $0.32 \pm 0.26$
& $0.65 \pm 0.41$
& $0.62 \pm 0.44$
& $1.52 \pm 1.27$ \\

GA w/ Teacher-Forcing
& $0.30 \pm 0.24$
& $0.60 \pm 0.39$
& $0.56 \pm 0.40$
& $1.43 \pm 1.05$ \\

GA w/ Free-Running
& $\mathbf{0.26 \pm 0.20}$
& $\mathbf{0.49 \pm 0.35}$
& $\mathbf{0.46 \pm 0.33}$
& $\mathbf{1.28 \pm 0.82}$ \\
\bottomrule
\end{tabular}
\end{table*}

\section{Additional Qualitative Results}
\label{app:qualitative_appendix}

This appendix provides additional qualitative evidence for the two levels of evaluation used in the main paper.
We first show failure cases of the baseline model, including weak rollout-level robustness under identical actions and local violations of the group-action conditions.
We then provide additional examples after GA training, showing that the improvements measured by GAC and GAR are also visible in generated rollouts.

\begin{figure*}[!ht]
\centering
\includegraphics[width=\textwidth]{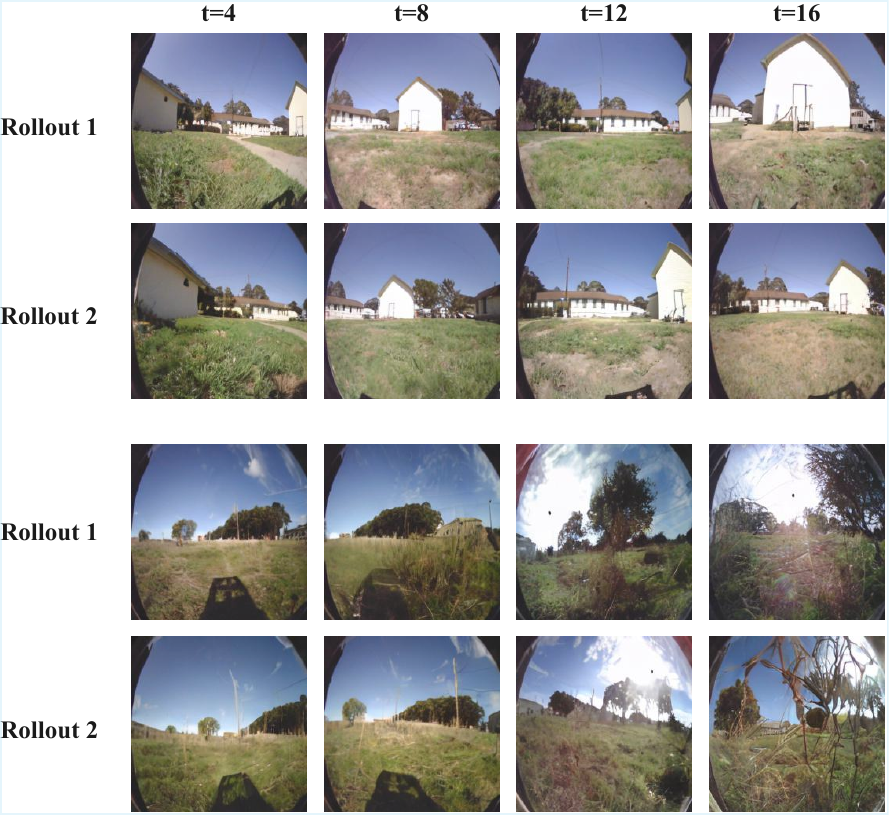}
\caption{
Additional qualitative examples of weak rollout-level robustness in the baseline NWM.
Each example shows two stochastic rollouts generated from the same initial observation and identical action sequence.
Although the frames can remain visually plausible, the induced motion may diverge over time.
}
\label{fig:traj_more}
\end{figure*}

\subsection{Baseline Rollout Robustness Failures}
\label{app:rollout_more}

Figure~\ref{fig:traj_more} shows additional qualitative examples of weak rollout-level robustness in the baseline NWM.
Each example compares two stochastic rollouts generated from the same initial observation and identical action sequence.
Although individual frames remain visually plausible, the induced motion can diverge noticeably over time, leading to inconsistent trajectory behavior across repeated rollouts.

These examples complement the GAR results in Sec.~\ref{sec:exp_gar}.
They show that the dispersion measured by $\mathcal{E}_{\mathrm{GAR}}$ is directly visible in generated sequences rather than being only a property of recovered trajectory metrics.
The same behavior appears across multiple cases, consistent with the quantitative GAR results.

\begin{figure*}[!ht]
\centering
\includegraphics[width=\textwidth]{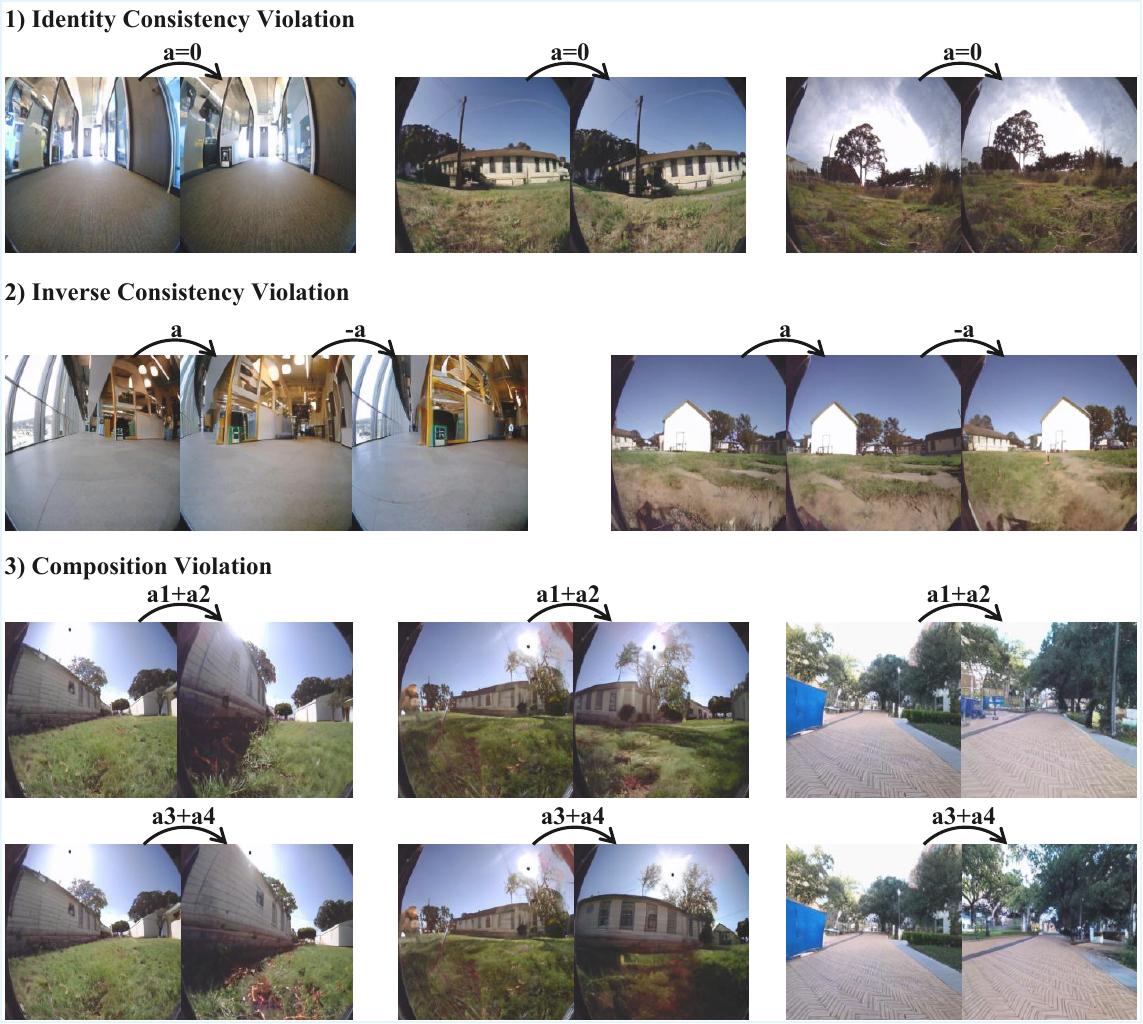}
\caption{
Additional qualitative examples of local GAC failures in the baseline NWM.
The figure includes representative cases for identity consistency, inverse consistency, and local composition consistency.
Observed behaviors include drift under zero-action probes, incomplete recovery after forward--inverse action sequences, and mismatched outcomes under locally equivalent action decompositions.
}
\label{fig:qual_more}
\end{figure*}

\subsection{Baseline Local GAC Failures}
\label{app:qualitative_cases}

Figure~\ref{fig:qual_more} shows additional qualitative examples of local GAC failures in the baseline NWM.
The examples cover identity consistency, inverse consistency, and local composition consistency.
They make the three local failure modes visible in generated rollouts, complementing the quantitative GAC errors in Sec.~\ref{sec:exp_gac}.

For identity consistency, inserted zero-action segments do not reliably preserve the current state.
Even when the action input indicates no motion, the generated rollout may still exhibit visible drift or viewpoint change.
For inverse consistency, a local action segment followed by its inverse often fails to recover the starting state, leaving a residual offset after the forward--inverse sequence.
For local composition consistency, locally equivalent action decompositions can lead to noticeably different rollout outcomes, indicating sensitivity to temporal realization beyond cumulative displacement alone.

Together, these examples show that the local errors measured by $\Delta_{\mathrm{id}}$, $\Delta_{\mathrm{inv}}$, and $\Delta_{\mathrm{comp}}$ correspond to concrete failures in generated action-conditioned behavior.
They also support the interpretation that poor rollout robustness is connected to violations of local group-action conditions.

\begin{figure*}[!ht]
\centering
\includegraphics[width=\textwidth]{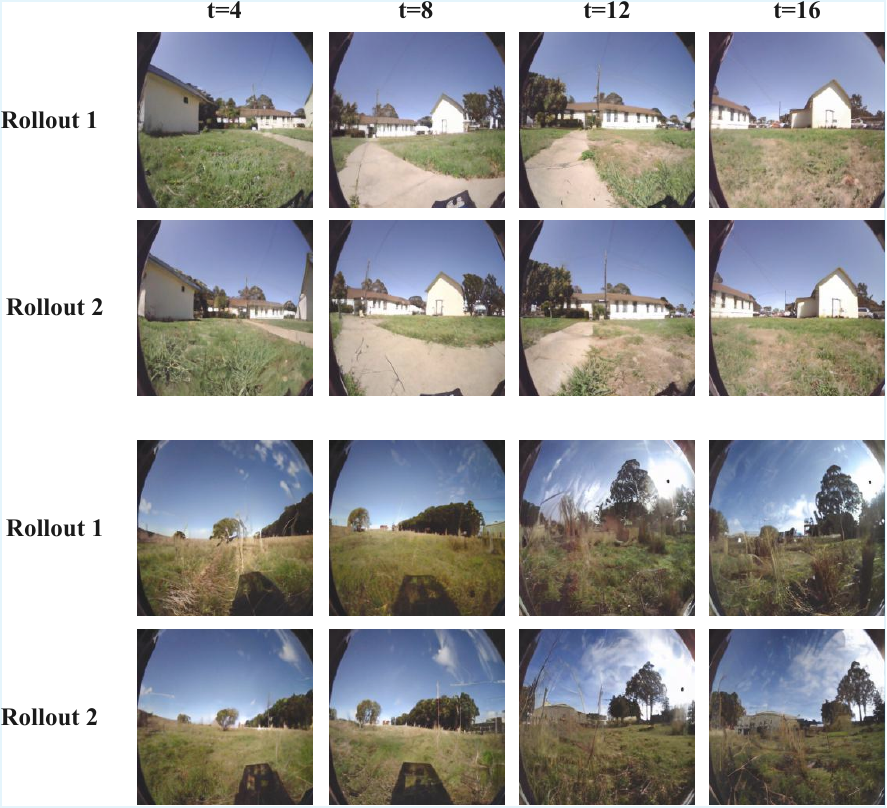}
\caption{
Additional qualitative examples of repeated-rollout robustness after GA training.
Each example shows stochastic rollouts generated from the same initial observation and identical action sequence.
Although visual details may vary across samples, the induced motion remains more consistent.
}
\label{fig:traj_more_consistency}
\end{figure*}

\begin{figure*}[!ht]
\centering
\includegraphics[width=\textwidth]{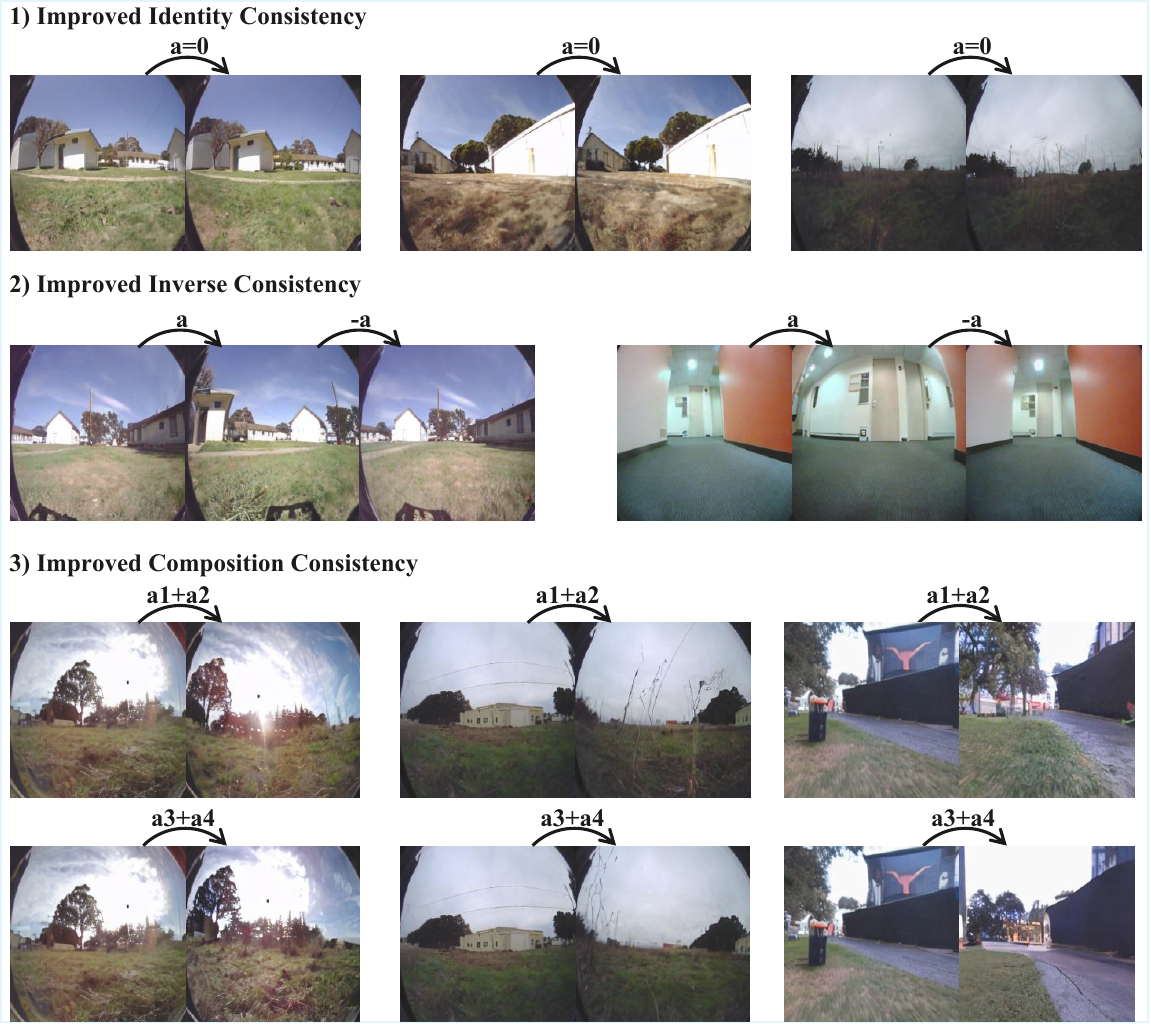}
\caption{
Additional qualitative examples of local group-action behavior after GA training.
The examples show reduced drift under zero action, approximate recovery under forward--inverse action sequences, and more compatible outcomes under locally equivalent action decompositions.
}
\label{fig:qual_probe_more}
\end{figure*}

\subsection{Additional Qualitative Results after GA Training}
\label{app:qual_after_training}

Figures~\ref{fig:traj_more_consistency} and~\ref{fig:qual_probe_more} provide additional qualitative examples after GA training.
Figure~\ref{fig:traj_more_consistency} shows repeated stochastic rollouts under identical initial observations and action sequences.
Although individual samples may still differ in visual details, their motion evolution remains more consistent across rollouts.
This complements the GAR improvements reported in Sec.~\ref{sec:exp_gar} and Appendix~\ref{app:gar_more}.

Figure~\ref{fig:qual_probe_more} shows representative local group-action behavior under identity, inverse, and composition transformations.
The GA-trained model exhibits limited drift under zero action, approximate recovery after forward--inverse action sequences, and compatible outcomes under locally equivalent action decompositions.
These examples complement the GAC improvements in Sec.~\ref{sec:exp_gac} and Appendix~\ref{app:gac_more}.

Taken together, the qualitative results show both sides of the empirical story.
The baseline model can generate visually plausible videos while violating local group-action conditions and diverging under repeated rollout.
After GA training, the generated rollouts remain more consistent with the action-composition structure, and the improvement is visible not only in aggregate metrics but also in the generated sequences themselves.

\section{Related Work}
\label{sec:appendix_relatedwork}

\subsection{Visual Navigation and Embodied Action Prediction}

Learning-based visual navigation studies how agents select actions from visual observations, goals, or task specifications.
This problem has been widely explored in goal-conditioned navigation~\cite{sridhar2024nomad,vaidheeswaran2025goal,shen2026efficient}, vision-and-language navigation~\cite{gu2022vision,zhou2024navgpt,xue2026profocus}, and foundation-model-based navigation~\cite{shen2025effonav,moriconi2025foundation,khan2025foundation}.
Early learning-based approaches typically combined visual perception with policy learning for goal reaching or waypoint prediction.
More recent methods increasingly adopt large-scale data, end-to-end training, and expressive action distributions to improve generalization across scenes and embodiments.
Representative examples include GNM~\cite{shah2023gnm} and ViNT~\cite{shah2023vint}, which learn generalizable visual navigation policies from diverse data, and NoMaD~\cite{sridhar2024nomad}, which uses diffusion modeling to represent multi-modal navigation actions.

Recent work further pushes navigation toward longer horizons, stronger semantic grounding, and broader task generalization.
LH-VLN~\cite{song2025longhorizonvln} studies long-horizon vision-language navigation with new benchmarks and evaluation protocols, while NavFoM~\cite{zhang2025navfom} scales navigation modeling across tasks and embodiments.
Large-model-based navigation methods have also become increasingly common.
VLN-R1~\cite{qi2025vlnr1} explores reinforcement fine-tuning for vision-language navigation, and ImagineNav++~\cite{wang2025imaginenavpp} uses imagined future views and large vision-language models to improve planning and reasoning.
Other recent methods investigate foundation-model-based action selection, goal-distance estimation, and proactive prediction for navigation~\cite{milikic2025vld,zhai2026vnm,liu2025navforesee}.

These methods are primarily policy-centric.
Their main objective is to map observations, instructions, or goals to actions that complete a navigation task.
The present work studies a different but complementary question.
Rather than asking whether a model predicts a good action, it asks whether an action-conditioned world model produces rollouts whose induced dynamics remain faithful to the composition structure of the actions themselves.
This distinction matters because a policy may be useful for goal reaching, and a generated rollout may appear visually plausible, while the underlying motion evolution still fails to preserve identity, inverse, and composition relations under repeated rollout.
Our focus is therefore not visual navigation policy learning, but the structural action faithfulness of action-conditioned dynamics.

\subsection{Latent and Diffusion-Based World Models}

The study of action-conditioned dynamics is closely related to model-based reinforcement learning, where learned transition models are used for planning, policy learning, and long-horizon decision making.
Latent world models such as Dreamer and its successors~\cite{hafner2019dream,hafner2023mastering} demonstrate that imagined rollouts in a compact latent space can support efficient control across diverse domains.
Related work studies reconstruction-free latent imagination~\cite{okada2020dreaming}, memory-augmented latent dynamics~\cite{mu2021mbrlimagination}, and recent applications of latent world models to embodied planning and autonomous driving~\cite{jin2026dreamerad,li2026latent4dplanning}.
These methods establish that learned latent dynamics can be useful for downstream control and planning, and they motivate the use of latent rollout as an efficient computational substrate.

However, the objectives in model-based reinforcement learning are usually organized around control performance, sample efficiency, or planning utility.
They do not explicitly ask whether the learned transition respects the algebraic relations of the action space.
In contrast, the present work focuses on a structural property of the rollout dynamics themselves.
The question is not only whether latent imagination is useful, but whether the action-conditioned transition remains compatible with the composition, identity, and inverse relations induced by physical motion.

Modern visual world models are also shaped by the rapid development of diffusion-based image and video generation.
Denoising diffusion models~\cite{ho2020ddpm} and latent diffusion models~\cite{rombach2022ldm} provide scalable foundations for high-fidelity visual synthesis.
Video diffusion models extend this paradigm to temporal prediction and controllable motion, including early video diffusion work~\cite{ho2022videodiffusion,ho2022imagenvideo}, latent video generation~\cite{blattmann2023videoldm}, motion adaptation~\cite{guo2024animatediff}, Stable Video Diffusion~\cite{blattmann2023svd}, and large-scale video generators such as CogVideoX~\cite{yang2024cogvideox} and Open-Sora~\cite{zheng2024opensora}.
As these models improve in fidelity, duration, and temporal coherence, they become increasingly viable backbones for action-conditioned visual rollout.

Several recent systems adapt pretrained or large-scale video diffusion models to interactive or embodied settings.
Vid2World~\cite{huang2025vid2world} transfers pretrained video diffusion models into interactive world models, while DreamZero~\cite{kim2026dreamzero} builds a world-action model on top of a pretrained video diffusion backbone.
These systems reflect a broader trend: video generators are increasingly expected to function not only as visual synthesizers, but also as simulators conditioned on actions.
The present work builds on this transition but targets a different property from perceptual quality.
Visual fidelity, temporal smoothness, and controllability do not by themselves guarantee that generated motion is governed by the input actions.
Our focus is therefore not generative realism alone, but whether learned rollout dynamics preserve the local observable consequences of action composition.

\subsection{Embodied World Models for Navigation and Simulation}

World models approach embodied decision making by predicting future observations conditioned on candidate actions.
This paradigm has a long history in model-based reinforcement learning, where latent dynamics models such as Dreamer~\cite{hafner2019dream,hafner2023mastering} support planning and policy learning through imagined rollouts.
More recently, similar ideas have been extended to visual navigation and embodied simulation, where future-state generation is used not only for representation learning but also for planning, action selection, and counterfactual evaluation.

Navigation World Models~\cite{bar2025nwm} show that controllable video generation can support action-conditioned rollout and planning in navigation.
This line of work is especially relevant to our setting because navigation actions correspond to physical motion in the plane, making the structure of action composition directly meaningful.
Recent work further broadens embodied world modeling through geometry-aware prediction, language-conditioned future modeling, general-purpose simulation, and action-centered world-action modeling.
Aether~\cite{zhu2025aether} combines geometry-aware reconstruction, action-conditioned prediction, and visual planning in a unified framework.
Language-conditioned world modeling for visual navigation~\cite{dong2026lcvn} studies future prediction, language grounding, and action generation jointly.
UniSim~\cite{yang2023learning} explores large-scale video models as general-purpose embodied simulators.
Other systems such as UniWM~\cite{liu2025uniwm}, NavForesee~\cite{liu2025navforesee}, MoWM~\cite{hu2025mowm}, GigaWorld-Policy~\cite{liu2026gigaworldpolicy}, DreamZero~\cite{kim2026dreamzero}, and MTV-World~\cite{jin2025mtvworld} investigate unified world modeling, planning-oriented future imagination, or action-centered world-action modeling for embodied control.

These studies demonstrate the promise of world models for embodied decision making.
At the same time, visual plausibility and planning utility do not by themselves establish that learned dynamics faithfully follow actions.
A model may generate realistic-looking futures and still allow the same action sequence to induce inconsistent motion across stochastic rollouts.
It may also drift under zero action, fail to cancel a motion with its inverse, or produce incompatible endpoints under locally equivalent decompositions.
The issue examined in the present work is therefore more specific than embodied world modeling in general.
We ask whether action-conditioned rollout remains consistent with action composition, and how failures of this consistency appear as trajectory instability.
This places the proposed framework between generative world modeling and dynamical consistency analysis.

\subsection{Action Consistency and World-Model Evaluation}

As world models become more capable, their evaluation has expanded beyond perceptual metrics, reconstruction quality, and short-horizon prediction error.
Recent work has begun to ask whether generated worlds are useful for downstream decision making and whether they remain reliable under interaction.
WorldSimBench~\cite{qin2024worldsimbench} evaluates predictive video models from perceptual and embodied perspectives, including action-level consistency in downstream tasks.
WoW-World-Eval~\cite{fan2026wowworldeval} studies embodied world models through perception, planning, prediction, generalization, and execution.
WorldArena~\cite{shang2026worldarena} highlights the gap between perceptual quality and embodied functional utility, while WorldLens~\cite{liang2025worldlens} evaluates driving world models along multiple axes including action-following and downstream reliability.
These benchmarks reflect an important shift from asking whether generated videos look realistic to asking whether they function as useful and reliable simulations.

A closely related line of work focuses more directly on controllability, memory, and rollout reliability.
MIND~\cite{ye2026mind} evaluates action control and memory consistency in closed-loop rollouts.
Toward Memory-Aided World Models~\cite{wang2025memoryaidedwm} emphasizes spatial consistency under long-horizon navigation rollouts.
Recent consistency-oriented world modeling work~\cite{wang2026multitokenwm,gao2026wildworld} studies state alignment and action following beyond one-step prediction quality, while Mobile World Models~\cite{huang2026mwm} highlights the gap between visually plausible prediction and action-conditioned consistency and proposes strategies to improve multi-step rollout behavior.
Together, these studies show growing awareness that visual realism is insufficient for world modeling, especially when generated rollouts are used for planning or decision making.

The present work is aligned with this emerging direction but differs in how the problem is formulated.
Existing evaluations often treat action consistency as a benchmark dimension, an action-following score, or a global rollout symptom.
Such evaluations are valuable, but they do not explicitly formalize the local action relations that make long-horizon action-conditioned rollout stable.
In contrast, this work asks which local relations should be preserved for a navigation world model to be action-faithful.
The group-action view decomposes action consistency into identity preservation, inverse cancellation, and local composition consistency.
These properties are not arbitrary probes; they are local observable consequences of the group action induced by planar ego-motion.
The proposed GAC metric evaluates these local conditions, while GAR measures whether repeated stochastic rollouts remain robust under the same action sequence.
This provides a structured explanation of why repeated rollouts can diverge even when generated frames remain visually plausible.

\subsection{Positioning of the Proposed Work}

The proposed work is closest to embodied world modeling, action-conditioned video generation, and recent efforts on evaluating world-model consistency.
Existing studies have begun to move beyond perceptual realism by measuring spatial coherence, memory consistency, action controllability, or trajectory-level stability.
These directions reveal important failures of visually plausible simulators.
Our work shares this motivation, but places a different criterion on action-conditioned dynamics: rather than treating trajectory inconsistency only as a rollout-level outcome, we ask whether the model preserves the local action-composition structure from which stable rollouts should arise.

This criterion distinguishes the proposed framework from broad embodied world-model benchmarks and trajectory-consistency evaluations.
Trajectory dispersion measures whether repeated stochastic rollouts under the same action sequence remain close in recovered state space, but it does not specify which relations of the action space are violated.
We instead start from the group action of planar ego-motion and instantiate its local observable consequences through identity consistency, inverse consistency, and local composition consistency.
GAC probes these local relations under controlled transformations, while GAR measures whether their preservation translates into rollout-level robustness.

The contribution is therefore a group-action formulation of action faithfulness and its operationalization for visual world models.
The same local conditions that support GAC are also converted into latent rollout constraints, providing a practical training signal that improves action-conditioned rollout stability without changing the backbone architecture.
In this sense, the proposed work complements existing world-model evaluation efforts by connecting rollout-level instability to local action-composition structure and by showing that this structure can guide model improvement.

\section{Discussion on Scope, Broader Implications, and Impact}
\label{app:discussion}

The GA framework is intended as a structured criterion for action-faithful world modeling.
This paper instantiates it in embodied navigation, where actions correspond to local ego-motion in the plane and are naturally tied to the composition structure of $SE(2)$.
Identity consistency, inverse consistency, and local composition consistency are therefore not task-specific heuristics, but observable consequences of the group-action ideal under the navigation action space.
GAC evaluates whether these local consequences are preserved in recovered state space, while GAR measures whether repeated stochastic rollouts remain stable under the same initial observation and identical action sequence.
Together, they separate local action-structure preservation from its rollout-level manifestation.

The present instantiation is deliberately focused.
It studies planar navigation actions, local ego-motion increments, and short-window composition probes because they provide a direct setting for testing action composition in video rollout.
Other embodiments may call for richer action structures, such as exact $SE(2)$ or $SE(3)$ composition, non-holonomic constraints, contact dynamics, or embodiment-specific motion priors.
Likewise, recovered trajectories provide an interpretable motion-level state representation, while future evaluations could incorporate geometry, semantic maps, object states, contact events, or task outcomes.
The latent training objective follows the same scope: it provides a lightweight and architecture-agnostic way to encourage action-composition consistency, while leaving room for stronger instantiations through equivariant architectures, consistency-preserving latent transitions, planning-aware objectives, or closed-loop training.

Beyond navigation, the same principle may apply to world-modeling problems in which the target system has known transformation structure.
In robotics, autonomous driving, manipulation, and physical simulation, actions are often tied to rigid-body motion, contact, kinematic constraints, or interaction rules.
These settings suggest direct extensions through richer motion groups, object-centric states, contact-aware transitions, or closed-loop action sequences.
Scientific world models provide a more distant but instructive case, because their dynamics are often organized by explicit physical symmetries and constraints rather than by embodied control commands.
Recent generative-simulation work has modeled molecular dynamics trajectories~\cite{jing2024mdgen}, protein equilibrium ensembles~\cite{lewis2025bioemu}, all-atom protein dynamics in learned latent spaces~\cite{sengar2025beyondensembles}, protein conformational ensembles~\cite{wang2026macdiff}, and folding pathways~\cite{pathdiffusion2026}.

These scientific systems do not share the same action space as navigation, but many are governed by transformations that play an analogous structural role.
Rotational and translational symmetries are central to geometric deep learning for 3D structures~\cite{thomas2018tensorfield,satorras2021egnn}, while permutation invariance of identical particles or atoms is a basic requirement for molecular representations~\cite{satorras2021egnn,hoogeboom2022edm}.
These principles have also been incorporated into geometry-aware generative models for molecular conformations~\cite{xu2022geodiff}, protein backbones~\cite{yim2023framediff}, and molecular docking~\cite{corso2023diffdock}.
The connection should therefore be understood as a guiding principle rather than a direct transfer of the navigation protocol.
For each domain, the relevant states, transformations, and observable consistency probes would need to be defined according to the structure of the system.

The broader impact of this work is primarily conceptual.
It argues that action-conditioned world models should be viewed not only as video generators conditioned on action tokens, but as simulators whose generated dynamics should remain organized by the actions or transformations they receive.
This adds a structural dimension to world-model evaluation beyond visual realism and task performance.
The proposed metrics do not certify real-world deployment safety or replace closed-loop validation.
Rather, they expose action-faithfulness failures that may remain hidden under perceptual metrics or short-horizon prediction losses, providing an additional lens for analyzing world models before they are used in planning, embodied evaluation, or simulation-based testing.

\end{document}